\documentclass[runningheads]{llncs}
\usepackage{latexsym}
\usepackage{paralist}
\usepackage{amssymb}
\usepackage{soul}
\usepackage{enumitem}
\usepackage{mathtools}
\usepackage{xcolor}
\usepackage[all]{xypic}

\usepackage{paralist}

\usepackage{graphics}

\usepackage[belowskip=-15pt,aboveskip=0pt]{caption}

\usepackage{amssymb}
\begin{document}

\title{The dynamics of belief: continuously monitoring and visualising complex systems}
\titlerunning{The dynamics of belief}
%\author{Edwin J.\ Beggs\orcidID{0000-0002-3139-0983} \and
%John V.\ Tucker\orcidID{0000-0003-4689-8760}\footnote{Orcid id.\ 0000-0002-3139-0983 and 0000-0003-4689-8760 respectively}
\author{Edwin J.\ Beggs \and
John V.\ Tucker
%\author{Edwin J.\ Beggs\orcidID{0000-0002-3139-0983} \and
%John V.\ Tucker\orcidID{0000-0003-4689-8760}
% \and Third Author\inst{3}\orcidID{2222--3333-4444-5555}
}
\authorrunning{E.J.\ Beggs \& J.V.\ Tucker}
\institute{School of Mathematics and Computer Science,\\ Computational Foundry, Swansea University, \\Bay Campus, Fabian Way, \\Swansea, SA1 8EN, United Kingdom}

\maketitle

\begin{abstract}
The rise of AI in human contexts places new demands on automated systems to be transparent and explainable. We examine some anthropomorphic ideas and principles relevant to such accountablity in order to develop a theoretical framework for thinking about digital systems in complex human contexts and the problem of explaining their behaviour.  Structurally, systems are made of modular and hierachical components, which we abstract in a new system model using notions of \textit{modes} and \textit{mode transitions}. A mode is an independent component of the system with its own objectives, monitoring data, and algorithms. The behaviour of a mode, including its transitions to other modes, is determined by functions that interpret each mode's monitoring data in the light of its objectives and algorithms.  We show how these \textit{belief functions} can help explain system behaviour by visualising their evaluation as trajectories in higher-dimensional geometric spaces. These ideas are formalised mathematically by abstract and concrete \textit{simplicial complexes}. We offer three techniques -- a framework for design heuristics, a general system theory based on modes, and a geometric visualisation -- and apply them in three types of human-centred systems.
\end{abstract}

%%%%%%%%%%%%%%%%%%%%%%%%%%%%%%%%%%%%%%%%%%%%%%%%%%%%%
%%%%%%%%%%%%%%%%%%%%%%%%%%%%%%%%%%%%%%%%%%%%%%%%%%%%%
%%%%%%%%%%%%%%%%%%%%%%%%%%%%%%%%%%%%%%%%%%%%%%%%%%%%%
%%%%%%%%%%%%%%%%%%%%%%%%%%%%%%%%%%%%%%%%%%%%%%%%%%%%%

\section{Introduction} \label{intro}

\begin{quotation}
\noindent We cannot afford to continue playing catch-up regarding AI -- allowing its use with limited or no boundaries or oversight and dealing with the almost inevitable human rights consequences after the fact.

\rightline{\textit{Michelle Bachelet, UN High Commissioner for Human Rights, 2021}\cite{unhc}}
\end{quotation}

\noindent We consider the automation of human-centred systems, a subject made urgent and controversial by modern AI.  Starting with reflections on explainability and possible design heuristics, we propose a new general system based on new mathematical models of modes and mode transitions, and having geometric visualisations of what a system mode `believes' about its monitoring data. We will address these points in turn and then summarise the contributions.

%%%%%%%%%%%%%%%%%%%%%%%%%%%%%%%%%%%%%%%%%%%%%%%%%%%%%
%%%%%%%%%%%%%%%%%%%%%%%%%%%%%%%%%%%%%%%%%%%%%%%%%%%%%

\subsection{Automation and the challenge of explainability}\label{explainability}

Automation aims to reduce human activity and dependency in all sorts of tasks. To do this, decisions and actions that belong to people must be abstracted, specified precisely and designed into a relevant system that reduces human involvement and interventions. Automation in this sense has been expanding its nature and scope since the 18th Century \cite{Giedion1948}. Through the new applications and ambitions of AI, automation is on the rise in the physical world, where it is long established, and in human affairs, where it is novel and controversial.\footnote{For example, we see it in new `smart' buildings and products, autonomous vehicles and weapons, staff recruitment and monitoring, customer services, medical triage, treatment and aftercare, and surveillance and criminal detection.}  Under the guise of AI, in so many fields we are seeing a triumphant march of measurement, data collection, data classification, computation and autonomous decision making. Wherever there is automation there is a potential problem understanding the behaviour it produces:

\textit{Is the automation functioning as specified or expected? Are its specifications and expectations what are required or desired or allowed?  What parts might need attention and change?} 

In the case of automatic physical systems problems almost always can be discovered, understood and resolved in some way -- thanks to centuries of science, engineering and mathematics. In the case of people the situation is quite other. Our understanding of automation is poorer, challenged by people's individuality, their limitations, expectations, conventions, bias, diverse intangible experiences; it is challenged by cultural backgrounds and formalised jurisdictions, and so on.  Personal and social contexts refuse to yield to neat theories with their limited predictive power and many unintended consequences; personal and social conditions change all the time anyway. We are still learning how \textit{not} to manage people in highly structured social environments, such as companies, as a century of management science, after Frederick W Taylor, testifies \cite{Taylor1911,Wren2011}. To the system designers, maintenance engineers and owners, the specifications are a challenge to formulate, and can be expected to mutate on deployment for all sorts of reasons.

In this paper, we will address system design for automation in human contexts, almost from first principles. We will reflect on a few established criteria of importance to people and suggest some ideas and directions for technical research that reflect these criteria. Thus, the tone of this theoretical paper is anthropomorphic, sampling aspects of human thought for technical inspiration: human thinking has a record of working, and human thought is explainable to humans. Subsequently, we will introduce high level data structures and their visualisation, which can be applied, using mathematical ideas of belief, to monitoring the dynamics of decision making in automation. 

%%%%%%%%%%%%%%%%%%%%%%%%%%%%%%%%%%%%%%%%%%%%%%%%%%%%%
%%%%%%%%%%%%%%%%%%%%%%%%%%%%%%%%%%%%%%%%%%%%%%%%%%%%%

\subsection{System theory as a tool}\label{system_theory_tool}

In human affairs, AI urgently needs to become transparent and explainable because systems must be trustworthy and `easy' to understand at some level.  We address the problem of \textit{transparency}, which due to both safety and human rights considerations can no longer be treated as an \textit{after the fact} investigation \cite{unhc}. 
Given the likely future economic importance of people and machines working together
\cite{DauWil,Kissinger2021}, it is vitally important that automations can clearly reveal their state and intentions to human co-workers in realtime. 

For us, the key notion is \textit{system}. Today, the term is everywhere in everyday speech. It was elevated into a scientific concept by biologist Ludwig Von Bertalanffy (1901-1972) starting in the late 1920s. Following World War II, system theories were developed in all sorts of directions notably in the engineering, management and economic sciences \cite{Bertalanffy1950,Bertalanffy1972}. Many of these theories were mathematical and had wide scope and grand ambitions \cite{Kalmanetal1969,Mesarovic1970,Mesarovic1975}. Computer science generated many system models with well-developed theories, including foundational work in highly abstract mathematical forms that attempted to bring formal unity. 

Indeed, general system theory has identified a remarkable collection of commonalities.\footnote{These include: basic definitions, system thinking, system topologies, life cycles, system performance, conceptual design, current state evaluations, solving methods, creative solutions, system synthesis, system analysis, optimization, solution assessment, virtual optimizing, system engineering, and evaluation of knowledge in economy and society.} In particular, system theory has long recognised that systems are modular and hierarchical and the components are a mix of equipment, software \textit{and} people \cite{Simon1962}. Any of these can --  and in practice do  -- fail. Add to this the threat of cyberattack, where the objective is to make systems fail. Absolute security or predictability of a system is impossible, so plans need to be made to contain and work around damage, the classic problem of \textit{resilience}.

%%%%%%%%%%%%%%%%%%%%%%%%%%%%%%%%%%%%%%%%%%%%%%%%%%%%%
%%%%%%%%%%%%%%%%%%%%%%%%%%%%%%%%%%%%%%%%%%%%%%%%%%%%%

\subsection{A new general system model}\label{new_general_system-model}

We introduce an abstract notion of \textit{mode} of a system designed to describe the structure and functioning of the system, and a geometric realisation of the modes to allow visualisation and implement continuity. The description in terms of modes would be accessible on demand with the visualisation showing not only which mode the system is in, but whether it is fit-for-purpose, and to which mode it is likely to transition in the near future \cite{shnAcc}.

Real world systems are driven by data, which may not be reliable. The decisions and actions taken by the system are based on \textit{judgements} about the data available against known criteria, which we are going to refer to as \textit{beliefs}. Thus, as far as mathematical models are concerned, the behaviour of the system over time is a trajectory in a hugely complex idealised state space. For this to be transparent we need to reveal a dynamics of belief that explain the dynamics of system behaviour.  In short, this dynamics of belief is a new trajectory in a space of beliefs about data and perceptions of the system.

We consider a large range of data sources as simply \textit{oracles}, which are providers of information which might, or might not, be reliable. Typically, these are gathered by sensors and generated by computations external to the system (e.g., by neural nets or quantum devices). We also include outputs sent to external devices as oracles (as many actuators talk back). By restricting a mode's access to various oracles we have a high-level mechanism for enforcing restrictions on what a mode can know and do, i.e., a method for enforcing transparency. 

The mode construction and its visualisation using simplicial complexes were originally designed for complex physical control systems \cite{BeTuSimp}. In this paper we have deliberately moved away from that to social examples to develop further the generality of the concepts and methods. Here, for example, we have also defined the general idea of a \textit{frame of belief} which has relevance to social and academic constructs and explainablity. The idea is to present the modes and their visualisation as an architecture in logical AI. 

%%%%%%%%%%%%%%%%%%%%%%%%%%%%%%%%%%%%%%%%%%%%%%%%%%%%%
%%%%%%%%%%%%%%%%%%%%%%%%%%%%%%%%%%%%%%%%%%%%%%%%%%%%%

\subsection{Contributions and structure of this paper}\label{structure}

We present a general approach with an informal characterisation of modes and mode transitions, algebraic and geometric models, and examples to illustrate their richness.
Main contributions of this paper and their location are:

1. To consider some aspects of human behaviour informally, before formalising a model of how the knowledge and beliefs of a system and its responses change over time (Section \ref{explainability}). 

2. To define concepts and principles for modelling systems using modes and mode transitions (Section \ref{modes}).

3.  To address explainability by defining a mathematical model for beliefs and visualising changes in beliefs about the monitoring data and hence the system, in modes using simplicial complexes (Sections \ref{simplicial_complexes} and \ref{belvis4}). 

4. To illustrate the combination of modes and beliefs in three case studies of human-centred systems (Section \ref{social_examples}).

5. To highlight the role of change in underlying epistemic frameworks of belief in understanding the design and operation of systems (Section \ref{belief_theory_change}).

%%%%%%%%%%%%%%%%%%%%%%%%%%%%%%%%%%%%%%%%%%%%%%%%%%%%%
%%%%%%%%%%%%%%%%%%%%%%%%%%%%%%%%%%%%%%%%%%%%%%%%%%%%%
%%%%%%%%%%%%%%%%%%%%%%%%%%%%%%%%%%%%%%%%%%%%%%%%%%%%%
%%%%%%%%%%%%%%%%%%%%%%%%%%%%%%%%%%%%%%%%%%%%%%%%%%%%%

\section{Explainability and an anthropic view of systems} \label{explainability}

We will briefly review some ideas that lie just below the surface of our view of systems. We will comment on explanations and how people's performance in a situation or scenario may be systematised.

%%%%%%%%%%%%%%%%%%%%%%%%%%%%%%%%%%%%%%%%%%%%%%%%%%%%%
%%%%%%%%%%%%%%%%%%%%%%%%%%%%%%%%%%%%%%%%%%%%%%%%%%%%%

\subsection{What is an explanatory framework?}\label{framework_models}

Our considerations lead to questions about the structure of the system, its design and deployment, and its operation and performance: What counts as an explanatory framework? What ways might there be of seeing what is going on? To develop an informal \textit{explanatory framework}, consider some general principles, starting with:

\medskip
\begin{description}[noitemsep,topsep=0pt,align=left]
\item [Observation.] An explanation depends upon on data available to do with the system; the data collected, processed, computed by and about the system.
 
\item [Categorisation.]  An explanation is based on categorisations of possible behaviours that constitutes a conceptual framework for the human coworker.

\item [Refinement.]  An explanation is capable of refinement, revealing more detail, according to the significance of the information available.

\item [Prediction.] An explanation is capable of making predictions about the system is certain situations. 
\end{description}
\medskip
Tightly coupled to the explanatory framework is a \textit{visualisation framework} ideally satisfying:

\medskip
\begin{description}[noitemsep,topsep=0pt,align=left]
\item [Continuity.] The behaviour over time should be understandable from the visualisation.

\item [Faithfulness.] The visualisation provides access to the actual functioning of the system, and does not conceal important properties of its function. 
\end{description}

\medskip

One way this can be done is to incorporate the explanatory and visualisation frameworks as an essential part of the design process, allowing us to `see' what the system is `thinking'.

Of course, this explanatory framework is impossible without simplification and information hiding on a grand scale. 

However, the frameworks have foundations that are theoretical and empirical and,  indeed, influenced by the cultural, social and political. This is evident in human centred systems.

Since explainability is defined in terms of human understanding, and humans have a record of being resilient, we begin by considering how humans might understand and control a situation systematically. Our purpose is to seek ideas that can be abstractly formulated and used to shape a conceptual framework for explainability and our technical proposals.

%%%%%%%%%%%%%%%%%%%%%%%%%%%%%%%%%%%%%%%%%%%%%%%%%%%%%
%%%%%%%%%%%%%%%%%%%%%%%%%%%%%%%%%%%%%%%%%%%%%%%%%%%%%

\subsection{Designing scenarios}\label{scenarios}

Consider a person responsible for a system and imagine him or her in a scenario composed of the following:

\medskip
\begin{description}[noitemsep,topsep=0pt,align=left]
\item [Objectives.]  A person has long term \textit{aims} which entail shorter term \textit{objectives}.
\item [Modes.]  The system is in a state and they are presented with a choice of possible actions.
\item [Principles.]  There are `principles' -- assumptions of various kinds -- implying that actions have certain consequences, good or ill.
\item [Intent.]  They intend to take the actions whose consequences they believe will best further their objectives. 
\end{description}
\medskip

\noindent We need to have clear the distinction between between 

(a) \textit{principles} about a real situation 

(b) \textit{observations} of the situation 

(c) \textit{actions} applied to the situation. 

\noindent In science and engineering these ideas are familiar; for instance, in ballistics, the objective is to hit a target. Newton's laws of motion, some fluid dynamics, would be a source of principles; observation would be measurements of the wind, humidity, temperature and range of an object to hit a target; and an action by the gunner would be to set the parameters to aim to hit the target and fire. Human experience is also a factor providing essential heuristics, of course. In more complex situations the set of principles is a large set of assumptions explicit and implicit only some of which will shape the systematisation.

\medskip
\begin{description}[noitemsep,topsep=0pt,align=left]
\item [Oracles.] A person who is uncertain will \textit{observe} more information about the real sitiuation.
\item [Beliefs.] Using their \textit{principles}, they will evaluate the information to form \textit{beliefs} about how possible actions best further their objectives. 
\item [Choice.] Based on these beliefs they choose actions to perfrom.
\item [Aims.] They may change objectives based on aims and  information about reality.
\end{description}
\medskip

An objective may merely be the legacy of a past decision. For instance, a general decides to capture a hill, and from then onwards many decisions are taken using this as an objective. If capturing the hill proves impossible, the general will likely need to choose another objective. In turn, the aims within which the general operates may have been set by a dictator deciding to invade another country. 

When humans design or control systems they have an opportunity to specify objectives which the systems cannot override. For example, the average electronic heating controller cannot decide to vary the temperature outside a specified range (unless an error happens). At an extreme of the AI spectrum, Isaac Asimov's fictional robots were programmed with three laws of robotics as immutable objectives which would override any other objectives.

%%%%%%%%%%%%%%%%%%%%%%%%%%%%%%%%%%%%%%%%%%%%%%%%%%%%%
%%%%%%%%%%%%%%%%%%%%%%%%%%%%%%%%%%%%%%%%%%%%%%%%%%%%%

\subsection{Designing components for actions}\label{actions}

\begin{quotation}
\noindent
Frustra fit per plura quod potest fieri per pauciora. 

\noindent [It is futile to do with more things that which can be done with fewer.] 

\rightline{\textit{William of Ockham, Summa Totius Logicae, circa 1323 A.D.}}  
\end{quotation}

\noindent Consider two guidelines for people on how to organise the set of possible actions in a given scenario, one of these addressing \textit{simplicity} (and, so, indirectly explainability) and the other addressing \textit{resilience}, when current assumptions fail.

\medskip
\begin{description}[noitemsep,topsep=0pt,align=left]
\item [Occam's Razor 1: Actions.] A person will prefer a smaller number of possible actions rather than a larger number if they believe the smaller number can bring their intent about.
\item [Incompleteness 1: Actions.]  If they believe that the possible actions will not bring about their objectives, they seek more possible actions.
\end{description}
\medskip

Thus, we reduce the set of possible actions when we can, and expand it when we have to. 

Of course, these two guidelines can also be applied to the principles we use to design and reason about the system. 

\medskip
\begin{description}[noitemsep,topsep=0pt,align=left]
\item [Occam's Razor 2: Principles.] In a given scenario, use the least number of principles required to explain reality and model the objectives.
\item [Incompleteness 2: Principles.]   When reality becomes inconsistent with the principles, seek other principles. 
\end{description}
\medskip

We can also consider Occam's Razor as a guide to hiding irrelevant data. Information hiding is an idea central to software engineering, formulated explicitly as a methodological principle by David Parnas in the 1970s and brought into mainstream software development in abstract data types and object oriented programming languages, long before our current data-overloaded society. In societies, people have specialised roles, often related to data hiding:

\medskip
\begin{description}[noitemsep,topsep=0pt,align=left]
\item [Specialisation.] People only see a minimal amount of information for a role which they do not perform.
\item [Restricted access.] Only certain people have a right to access certain data to maintain security or privacy.
\end{description}
\medskip

Of course, in applying these points to automatic systems we have to make allowances. The average electronic heating controller has no idea of the laws of convection or conduction, it only knows the formulae it has been programmed with. However, even it has some understanding of incompleteness -- e.g., when its actions seem not to work --  it uses its fault light and hopes for an intervention by a technician. 

The big problem with incompleteness is that the person's existing knowledge may not have a previously prepared method to achieve the objectives -- this is the acquiring of new skills. The person has to be able to make new possible courses of action, which need to be checked, first purely theoretically and then more practically.

\begin{description}[noitemsep,topsep=0pt,align=left]
\item [Innovation.] The generation of new possibilities.
\item [Imagination.] New possibilities are matched against existing principles to predict their effect.
\item [Testing.] Controlled experiments are used to validate new possibilities.
\end{description}

In the rest of this paper we shall use these ideas to suggest and shape formal concepts that make up a possible formal complement to this informal framework. In particular, we consider how to both formalise and visualise scenarios, using a dynamics generated by beliefs about the basic monitoring data available to the system. The structure of our thinking is, essentially, to formalise this \textit{critical path}:

\begin{eqnarray}\label{basiceq}
 \mathrm{Monitoring Data}\ \stackrel{\mathrm{Evaluate}} \longrightarrow \ \mathrm{Belief\ about\ data}\  \stackrel{\mathrm{Visualise}}  \longrightarrow \ \mathrm{Visualisation\ of\ belief} \ \
\end{eqnarray}
 
%%%%%%%%%%%%%%%%%%%%%%%%%%%%%%%%%%%%%%%%%%%%%%%%%%%%%
%%%%%%%%%%%%%%%%%%%%%%%%%%%%%%%%%%%%%%%%%%%%%%%%%%%%%
%%%%%%%%%%%%%%%%%%%%%%%%%%%%%%%%%%%%%%%%%%%%%%%%%%%%%
%%%%%%%%%%%%%%%%%%%%%%%%%%%%%%%%%%%%%%%%%%%%%%%%%%%%%

\section{Modes and visualisation} \label{modes}

\begin{quotation}
My thesis has been that one path to the construction of a non-trivial theory of complex systems is by way of a theory of hierarchy. Empirically, a large proportion of the complex systems we observe in nature exhibit hierarchic structure. On theoretical grounds we could expect complex systems to be hierarchies in a world in which complexity had to evolve from simplicity. 

\noindent \rightline{\textit{Herbert Simon} \cite{Simon1962}}
\end{quotation}

%%%%%%%%%%%%%%%%%%%%%%%%%%%%%%%%%%%%%%%%%%%%%%%%%%%%%
%%%%%%%%%%%%%%%%%%%%%%%%%%%%%%%%%%%%%%%%%%%%%%%%%%%%%

Modes are a multi-level categorisation of the objectives and behaviours of systems.   
 A mode combines data relevant to the objectives and operation of the system with algorithms governing the behaviour and the external communication channels and protocols.

%%%%%%%%%%%%%%%%%%%%%%%%%%%%%%%%%%%%%%%%%%%%%%%%%%%%%
%%%%%%%%%%%%%%%%%%%%%%%%%%%%%%%%%%%%%%%%%%%%%%%%%%%%%

\subsection{What is a mode?}\label{what_is_ a_mode}

Here is a working definition to begin building the conceptual framework:
 
 \begin{definition} \label{mode_working definition}
Consider a system \em{Sys} operating in an environment  \em{Env}.
A {\em mode} of the system  \textit{Sys} is defined by these characteristics:

\noindent
1.  A mode is associated to, or determines, a subset of the possible states of the system.

\noindent
2.  A mode is designed to deliver on and meet objectives for the system when in these states.

\noindent
3. A mode consists of

(a) methods to input data from the environment  \em{Env} and to output data to \textrm{Env}.

(b) data types and algorithms for implementing the objectives.

\noindent
4. Methods to evaluate the performance of the mode against its objectives and, if necessary, choose and transfer to another mode.
\end{definition}

Thus, a mode is responsible for specific aspects of the system's performance, which may be relevant to the high-level, primary objectives and purposes of the system, or lower-level, localised technical services. It may be autonomous or a platform for accessing other modes. Each mode owns its relationship with the environment.

%%%%%%%%%%%%%%%%%%%%%%%%%%%%%%%%%%%%%%%%%%%%%%%%%%%%%
%%%%%%%%%%%%%%%%%%%%%%%%%%%%%%%%%%%%%%%%%%%%%%%%%%%%%

\subsection{Principles for designing modes}\label{basic_principles}

With the initial intuitions of Definition \ref{mode_working definition} in mind,  we can formulate some design principles to develop a conceptual framework for thinking about systems in terms of modes and mode transitions.

\medskip\noindent
\textbf{Completeness}. A set of modes for a system is a classification of the operation or behaviour of a system. At any time, a system can be in one, or more, modes. 

\medskip\noindent
\textbf{Composition}. When a system is in a number of modes then that situation itself constitutes a mode.

\medskip\noindent
\textbf{Component}. A set of modes for a system consists of (i) a set of \textit{basic modes} and (ii) \textit{joint modes} made by combining other modes. 

\medskip\noindent
\textbf{Localisation}. Each mode possesses its own data about its behaviour and environment. This monitoring data determines the mode's own state space and data types.

\medskip\noindent Thus, our knowledge of the system at any time is localised to the modes active at that time.

\medskip\noindent
\textbf{Quantification}. If a state of the system is meaningful for a number of modes then the relevance or suitability of these modes for the state of the system must be quantified, calibrated and interpreted.

\medskip\noindent
\textbf{Visualisation}.
With quantification and calibration comes the possibility of visualisation via geometric objects drawn to a scale and via derived qualitative diagrams.

\medskip\noindent
We shall visualise in space the basic modes by vertices, and the joint modes arising by combining them as lines, triangles, tetrahedra etc. The geometric objects we build are called \textit{simplicial complexes}. Simplicial complexes are made up of geometric pieces called \textit{faces}. 

\medskip\noindent
\textbf{Modes and faces}. In a successful system model using modes, \textit{every face of its simplicial complex is a mode and every mode is represented by a face}.

\medskip\noindent
Quantification can take the form of a position in the space representing the suitability of the mode for the state of the system, in a line or triangle or tetrahedron, etc.\ representing the joint mode. This position is computed by \textit{belief functions} that from the state of the system calculates measures of the relevance of other modes to the state.

\medskip\noindent
In the time evolution of the system we need to decide

(i) if, and when, a system should change from one mode to another, and 

(ii) which new mode should chosen.

\medskip\noindent
\textbf{Thresholds}. The transition out of one mode into another is governed by the results of the quantification and calibration. The decision to move to a new mode may be specified by numerical thresholds. Transition has these stages:

(a) the realisation that the mode is approaching its limitations

(b) the selection of modes that could be more appropriate

(c) triggers to choose and change to a new mode.

\noindent The belief functions, in computing relevance, are a means to trigger changes of mode.

\medskip\noindent
\textbf{Explanation}. The conceptual system of modes, mode evaluation functions, thresholds and mode transition functions, and their visual representation in simplicial complexes, can serve as an explanatory framework for the dynamical behaviour of automatic systems.

%%%%%%%%%%%%%%%%%%%%%%%%%%%%%%%%%%%%%%%%%%%%%%%%%%%%%
%%%%%%%%%%%%%%%%%%%%%%%%%%%%%%%%%%%%%%%%%%%%%%%%%%%%%

\subsection{What is visualised? Some simple examples}\label{examples}

 We consider some simple examples to begin to shape our informal ideas of a system of modes, the evaluation of beliefs and how they might be visualised. 
 
  \begin{figure}[htbp]
\begin{center}
 \unitlength 0.45 mm
\begin{picture}(130,51)(10,27)
\linethickness{0.3mm}
\put(20,35){\line(0,1){30}}
\linethickness{0.3mm}
\multiput(70,50)(0.24,0.12){167}{\line(1,0){0.24}}
\linethickness{0.3mm}
\put(110,30){\line(0,1){40}}
\linethickness{0.3mm}
\multiput(70,50)(0.24,-0.12){167}{\line(1,0){0.24}}
\linethickness{0.2mm}
\put(70,50){\line(1,0){37}}
\linethickness{0.2mm}
\put(113,50){\line(1,0){12}}
\linethickness{0.3mm}
\multiput(110,70)(0.12,-0.16){125}{\line(0,-1){0.16}}
\linethickness{0.3mm}
\multiput(110,30)(0.12,0.16){125}{\line(0,1){0.16}}
\linethickness{0.1mm}
\multiput(82,48)(0.22,0.12){117}{\line(1,0){0.22}}
\linethickness{0.1mm}
\multiput(92,44)(0.18,0.12){83}{\line(1,0){0.18}}
\linethickness{0.1mm}
\multiput(112,57)(0.12,-0.13){67}{\line(0,-1){0.13}}
\linethickness{0.1mm}
\multiput(112,48)(0.12,-0.12){42}{\line(1,0){0.12}}
\linethickness{0.1mm}
\multiput(100,40)(0.14,0.12){50}{\line(1,0){0.14}}
\linethickness{0.3mm}
\multiput(89,57)(0.24,0.12){42}{\line(1,0){0.24}}
\put(99,62){\vector(2,1){0.12}}
\linethickness{0.3mm}
\multiput(89,43)(0.24,-0.12){42}{\line(1,0){0.24}}
\put(99,38){\vector(2,-1){0.12}}
\linethickness{0.3mm}
\put(90,52){\line(1,0){10}}
\put(100,52){\vector(1,0){0.12}}
\linethickness{0.3mm}
\put(23,44){\line(0,1){10}}
\put(23,44){\vector(0,-1){0.12}}
\put(20,70){\makebox(0,0)[cc]{judgement of monitoring data}}

\put(20,30){\makebox(0,0)[cc]{intervention}}

\put(55,54){\makebox(0,0)[cc]{incoming}}

\put(56,46){\makebox(0,0)[cc]{calls}}

\put(128,70){\makebox(0,0)[cc]{police}}

\put(130,30){\makebox(0,0)[cc]{ambulance}}

\put(135,51){\makebox(0,0)[cc]{fire}}

\put(-16,54){\makebox(0,0)[cc]{triggering}}
\put(-16,46){\makebox(0,0)[cc]{an action}}

\put(174,54){\makebox(0,0)[cc]{triage for}}
\put(174,46){\makebox(0,0)[cc]{emergency services}}

\end{picture}
\setlength{\belowcaptionskip}{-25pt}
\medskip
\caption{Visualising a mode transition for a trigger mechanism and triage for emergency services}
\label{modpic}
\end{center}
\end{figure}
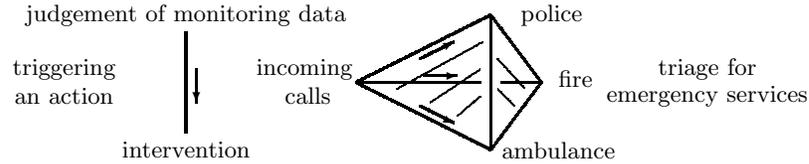

\begin{example}[A simple trigger]\label{man58}
 In Figure~\ref{modpic}, the left diagram represents visually a trigger for action that has three modes: two vertices and one line. Consider the line between a mode that makes judgements of some attribute using data monitoring the environment (the point at the top of the interval) and the mode that performs an intervention (the point at the bottom). Unlike a usual state transition graphs, where lines merely represent transitions or jumps, here the line \textit{itself} represents a (joint) mode where the decision to intervene unfolds according to its monitoring data. These observations are formed into a numerical representation of beliefs about the state of the system as points along the line, beginning at the top and moving to the bottom. Threshold points on the line begin to warn of, and then trigger, a change of mode; on reaching the bottom, the intervention begins. 
 
Witnesses can note the position along the line to see how things were going, and perhaps be reassured that the action was progressing correctly.
Such a point on a line indicator -- a progress bar -- is quite a common visualisation of estimated time taken for a task to be completed. Here, however, the line indicator is displaying the information (= an interpretation of the data) that the control system is actually using to make decisions. If the point is near the top, then a witness knows that the control system has no intention of intervening because they are looking directly at the system's own assessment of the situation. 
\qed
\end{example}
 
In this example, we see that a key to transparency is continuity. We do not discontinuously move between vertices on a graph, but have a continuous motion on a geometric visualisation, which allows an estimate of what the system is \textit{about to do and when}. Another key to transparency is reliability, that what we see is really related to the state of the system.  This means that \textit{the geometry of the visualisation is intimately related to how the system actually makes decisions}.

\begin{example}[A simple emergency] \label{man59}
In the second diagram of Figure~\ref{modpic}, we consider a notification warning of a possible incident -- such as personal attack, social disturbance, or collision of vehicles.  In general, the incident could involve some or all of the primary first responders of police, ambulance and fire brigade. These are represented by vertices in the diagram.  As more information becomes available over (hopefully a short) time period, relevant choices of services can be made and deployed to the scene. In particular, between the separate actions of say deploying police and ambulance we have a combined action of doing both simultaneously, and that is represented in the diagram by a line between the two vertices.

Altogether the 15 faces of the diagram represent 15 possible modes of the system!
The filled-in tetrahedron between the four vertices is the mode which gathers all available information and assesses the response. When beliefs about the state of the system (represented by a point in this tetrahedron or 3-simplex) reaches the triangle at the end then full deployment is instigated. If the incident in some way is misrepresented then the point moves, possibly returning to the warning vertex, and then to the rest of the system (not shown). The state when an action is actually taken should be one of the three action vertices two of the lines between them.
\qed
\end{example}

Systems frequently have to take account of several factors or carry out several tasks, and we take account of this by allowing combinations of modes to also form modes, which translates into edges and triangles, etc.\ representing modes.
In general, we shall represent beliefs about the state of a system as a point in a simplicial complex.

What exactly is a simplicial complex? A 0-simplex is just a point, and a 1-simplex is a line between some of the 0-simplices. Thus, a simplicial complex consisting of only 0 and 1-simplices is just a graph with edges the 1-simplices and vertices the 0-simplices. We shall not assume that this graph is directed -- if there are arrows they are either imposed by principles (e.g.,\ not being able to deploy a parachute which we have already jettisoned) or by the current objectives. 
Increasing the dimension, a basic 2-simplex is a filled in triangle and a basic 3-simplex is a solid tetrahedron, and so on.

The belief about the state of the system is visualised as a point in the geometric simplicial object, showing which mode the system is in and which mode is likely to come next. The position is given by the belief generated by algorithms on the basis of evidence. Technically, there is mathematical belief theory that is a generalisation of probability theory and, again, we refer to Section \ref{belvis4} for a discussion. 

As the modes are linked to the behaviour of the system, the visualisation of the modes gives the current behaviour of the system -- at least its intentions and its capabilities. 

Since each face of simplicial complex is a mode, the choice of a simplicial complex for a system can reveal the existence of possible modes that may be difficult to anticipate and account for (cf. Example \ref{man59}).

%%%%%%%%%%%%%%%%%%%%%%%%%%%%%%%%%%%%%%%%%%%%%%%%%%%%%
%%%%%%%%%%%%%%%%%%%%%%%%%%%%%%%%%%%%%%%%%%%%%%%%%%%%%
%%%%%%%%%%%%%%%%%%%%%%%%%%%%%%%%%%%%%%%%%%%%%%%%%%%%%
%%%%%%%%%%%%%%%%%%%%%%%%%%%%%%%%%%%%%%%%%%%%%%%%%%%%%

\section{Mathematical Tools 1: Simplicial complexes} \label{simplicial_complexes}

The three-stage critical path (\ref{basiceq}), from monitoring data to interpretation and belief to a visualisation of belief for our general conception of modes, can be modelled mathematically using abstract and concrete simplicial complexes, which we describe in this and the next section. 

%%%%%%%%%%%%%%%%%%%%%%%%%%%%%%%%%%%%%%%%%%%%%%%%%%%%%
%%%%%%%%%%%%%%%%%%%%%%%%%%%%%%%%%%%%%%%%%%%%%%%%%%%%%

\subsection{Simplicial complexes}

\begin{definition} \label{absc}
An abstract simplicial complex consists of a collection $\mathcal{C}$ of finite 
subsets of a set $\mathcal{M}$ with the property that if $Y\subset X\in \mathcal{C} $ then $Y\in \mathcal{C} $.
\end{definition}

In line with the \textit{Components Principle} in Section~\ref{modes}, we can take $\mathcal{M}$ to be the set of basic modes of the system, and $\mathcal{C}$ to be the set of joint modes. Thus, a basic mode becomes a singleton set $\{ \alpha  \}$ and a joint mode is a set $\{ \alpha_{1}, \ldots , \alpha_{k} \}$. The rule means that any subset of a mode is a mode.

We visualise this abstract simplicial complex as a \textit{geometric simplicial complex} made up of parts called \textit{simplicies}. This is done by taking
 each basic mode $\alpha\in \mathcal{M}$ to be a point (= a 0-simplex). Then the subsets of size two
 $\{\alpha,\beta\}\in \mathcal{C}$ are visualised as a line (= a 1-simplex) connecting the points $\alpha$ and $\beta$. So far, the resulting structure is an undirected graph, but now we continue in higher dimensions: 
 for example, the subsets of size three
 $\{\alpha,\beta,\gamma\}\in \mathcal{C}$ are visualised as a filled in triangle (= a 2-simplex) connecting the points $\alpha$, $\beta$ and $\gamma$, and so on. 
 
 \begin{proposition} \label{simpX}
To every abstract simplicial complex (as in Definition~\ref{absc}) is associated its standard realisation $\Delta_\mathcal{C}$, a {\em simplicial complex} as follows. Form a vector space $\mathbb{R}^\mathcal{M}$ with basis $e_\alpha$ for $\alpha\in \mathcal{M}$. The simplex spanned by $X\in \mathcal{C}$ is
\[
\Delta_X=\Big\{ \sum_{x\in X} \lambda_x  \, e_x : 
\lambda_x\in[0,1],\  \sum_{x\in X} \lambda_x=1\Big\}.
\]
The simplex $\Delta_X$ is a $(|X|-1)$-simplex where $|X|$ is the size of $X$, and if $Y\subset X$ then $\Delta_Y$ is a {\em face} of $\Delta_X$. 
\end{proposition}

\begin{example} \label{bvcu}
Suppose there are three basic modes $ \alpha, \beta, \gamma $. In general, the set of basic and joint modes in the 2-simplex (filled in triangle)
shown in Fig.~\ref{part4g}  is
$$\{  \alpha, \beta, \gamma, \alpha\beta, \alpha\gamma, \beta\gamma,  \alpha\beta\gamma \}.$$
\end{example}

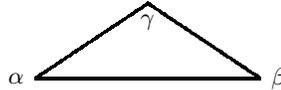
\begin{figure}[htbp]
\begin{center}

\unitlength 0.5 mm
\begin{picture}(95,20)(0,40)
\linethickness{0.3mm}
\multiput(30,40)(0.18,0.12){167}{\line(1,0){0.18}}
\linethickness{0.3mm}
\multiput(60,60)(0.18,-0.12){167}{\line(1,0){0.18}}
\linethickness{0.3mm}
\put(30,40){\line(1,0){60}}
\put(25,40){\makebox(0,0)[cc]{$\alpha$}}

\put(95,40){\makebox(0,0)[cc]{$\beta$}}

\put(60,55){\makebox(0,0)[cc]{$\gamma$}}

\end{picture}

\setlength{\belowcaptionskip}{-15pt}
\medskip\caption{Triangle for Example~\ref{bvcu}   }
\label{part4g}
\end{center}
\end{figure}

It is usually possible to re-draw the geometric simplicial complex in a smaller number of dimensions, the very large number in the Proposition simply gives an existence result.

%%%%%%%%%%%%%%%%%%%%%%%%%%%%%%%%%%%%%%%%%%%%%%%%%%%%%
%%%%%%%%%%%%%%%%%%%%%%%%%%%%%%%%%%%%%%%%%%%%%%%%%%%%%

\subsection{Evaluating states}

Consider the system from the viewpoint of its states. In line with the \textit{Quantification Principle} in Section~\ref{modes}, we would like to associate to a state of the system a point in the geometric simplicial complex. To do this it is easiest to consider a space $\mathbb{S}$ of possible states of the system. The basic mode $\alpha$ will apply when the system state is in a certain subset $U_\alpha\subset \mathbb{S}$ of all the possible states. If we want every possible state of the system to be assigned to a mode then the union of all the $U_\alpha\subset \mathbb{S}$ is all of $\mathbb{S}$, i.e.,\ we have a \textit{cover} of $\mathbb{S}$. 

\begin{definition} \label{local}
Let $\mathcal{M}$ be a set of basic modes.
We localise the state space to the states for the joint modes by defining $\mathbb{S}_X=\cap_{\alpha\in X} U_\alpha$ for $X\subset \mathcal{M}$ for which  $\mathbb{S}_X$ is not empty. Thus $\mathbb{S}_{\{\alpha,\beta\}}=U_\alpha\cap U_\beta\subset \mathbb{S}$.
\end{definition}

Now define
\[
\mathcal{C} = \big\{ X\subset \mathcal{M} : \cap_{\alpha\in X} U_\alpha\neq\emptyset\big\},
\]
and it is immediate that $ \mathcal{C} $ is itself an abstract simplicial complex.

\begin{definition} \label{realpu}
A {\em partition of unity} for the cover $U_\alpha\subset \mathbb{S}$ is a function $\phi_\alpha:\mathbb{S}\to[0,1]$ for every $\alpha\in\mathcal{M}$ such that 

(1)\quad
if $\phi_\alpha(s)\neq 0$ then $s\in U_\alpha$

(2)\quad $\sum_{\alpha\in\mathcal{M}}\phi_\alpha(s)=1$ for all $s\in \mathbb{S}$.

\noindent For specific circumstances we can impose extra conditions on $\phi_\alpha$, e.g.,\ continuity or computability. 
\end{definition}

\begin{definition} 
We now define a function $\underline\phi:\mathbb{S}\to \Delta_{\mathcal{C}}$
that classifies and visualises beliefs about the the states of a system using the simplicial complex $\Delta_\mathcal{C}$ by
\[
\underline\phi(s)= \sum_{\alpha\in\mathcal{M}} \phi_\alpha(s)\,e_\alpha\ .
\]

 \end{definition}
 
 \begin{example} 
 Figure~\ref{part4} visualises a simplicial complex and partition of unity for a cover by four sets.
 Note that the triangle in  Figure~\ref{part4} is shaded to form a 2-simplex precisely because $U_\alpha\cap U_\beta\cap U_\gamma$ is not empty. 
\end{example}

\begin{figure}[htbp]
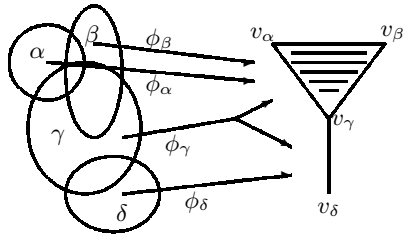

\begin{center}
\unitlength 0.5 mm


\setlength{\belowcaptionskip}{-15pt}
\medskip\caption{A partition of unity for a cover by four sets}
\label{part4}
\end{center}
\end{figure}

\subsection{Computations within modes}

\begin{figure}[htbp]
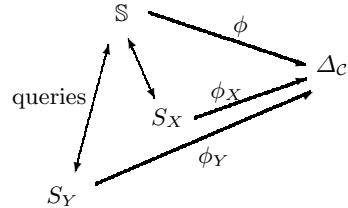

\begin{center}
\unitlength 0.7 mm
%
\setlength{\belowcaptionskip}{-15pt}
\medskip\caption{Localisation allowing computation of $\phi$}
\label{part83}
\end{center}
\end{figure}

 In general we may not know the state $s\in \mathbb{S}$ of the system. We may need to perform observations or measurements to do this. If we suppose that we are in mode $X\subset \mathcal{C}$ then \textit{oracles} bundled with $X$ are queried for information, and we need to compute with this data.

\begin{definition} \label{localdata}
For the mode $X$ we have the associated datatype $S_X$ containing, amongst other things, the data from queries (e.g.\ information from sensors). 
Then $S_X$ is our computable model of the state space $\mathbb{S}$ given that we believe that we are in mode $X$. We use a computable function $\phi_X:S_X \to \Delta_{ \mathcal{C} }$ to approximate our idealised function $\underline{\phi}:\mathbb{S}\to \Delta_{ \mathcal{C} }$. Fig.~\ref{part83} illustrates the actual computation ot this approximation to $\underline{\phi}:\mathbb{S}\to \Delta_{ \mathcal{C} }$.
\end{definition}

The function $\phi_X:S_X \to \Delta_{ \mathcal{C} }$ is used to see when we need to change mode from $X$ to another mode. If there is uncertainty about the current mode (as opposed to a normal mode transfer), then we could simultaneously run queries for other modes (e.g.\ $Y\subset  \mathcal{C} $ in Fig.~\ref{part83}). 

Decisions to prepare for a change of mode or to make a change of mode are based on positions in the simplicial complex. When in mode $X$ with state $s$ the interpretation of $X$'s monitoring data, represented by $\phi_X(s)\in\Delta_\mathcal{C}$, can trigger actions according to where $\phi_X(s)$ lies in appropriately defined subsets or zones. In particular, a change of mode is triggered by some exit zones. Recall the definition of thresholds in Section~\ref{basic_principles}.

%%%%%%%%%%%%%%%%%%%%%%%%%%%%%%%%%%%%%%%%%%%%%%%%%%%%%
%%%%%%%%%%%%%%%%%%%%%%%%%%%%%%%%%%%%%%%%%%%%%%%%%%%%%
%%%%%%%%%%%%%%%%%%%%%%%%%%%%%%%%%%%%%%%%%%%%%%%%%%%%%
%%%%%%%%%%%%%%%%%%%%%%%%%%%%%%%%%%%%%%%%%%%%%%%%%%%%%

\section{Three socially-centred examples}\label{social_examples}

So far we have mentioned physical control systems to illustrate the ideas surrounding modes. Here we give some social situations which can be modelled similarly. 
To relate them to the first part of the paper we need to consider the function $\underline\phi:\mathbb{S}\to \Delta_{\mathcal{C}}$ given after Definition \ref{realpu}.
We assess the idealised state space, the data which allows us to localise our position in that state space, and how $\phi$ depends on that data.

%%%%%%%%%%%%%%%%%%%%%%%%%%%%%%%%%%%%%%%%%%%%%%%%%%%%%
%%%%%%%%%%%%%%%%%%%%%%%%%%%%%%%%%%%%%%%%%%%%%%%%%%%%%

\subsection{An offender monitoring problem} \label{offender_monitoring_problem}

\noindent \textit{Scenario and its visualisation}. 
Consider an offender who has been sentenced for a crime committed under the influence of alcohol and who is
 being monitored for blood alcohol concentration and also tagged to limit how far they may go from home. For both alcohol and position we can assign a number, $x_{\mathrm{alc}}$ and $x_{\mathrm{tag}}$ respectively, in $[0,1]$ where $0$ is clearly OK and $1$ is clearly a problem. Given the context of this monitoring problem, the system state is an ordered pair $(x_{\mathrm{alc}},x_{\mathrm{tag}})\in [0,1]^2=\mathbb{S}$. This context deliberately omits much information, such as whether the offender is reading or sleeping, which if it was available is simply not recorded. The simplicity of assigning $\mathbb{S}=[0,1]^2$ is justified by the simplicity of the legal restrictions and the privacy of the offender. 
 
If at any time the following conditions are satisfied, interventions are triggered:

\smallskip

1) If the position is out of the allowed area, the offender is asked to see their probation officer.

2) If alcohol is higher than a trigger amount, the offender is asked to see their counsellor.

3) If both the above, the police are called as an incident may occur. 

\smallskip For both $x_{\mathrm{alc}}$ and $x_{\mathrm{tag}}$ we assign a region of values which are not problematic, and a region of values which are a problem, say $[0,\frac14]$ and $[\frac34,1]$ respectively. The fact that there is a gap between these regions is important: 

\smallskip 

a) It allows for error in a device or a difference in alcohol metabolism.

b) It allows for a warning to be sent to the offender, so that action can be taken to avoid an intervention (e.g.\ going home or stopping drinking). 

\medskip
\noindent\textit{Evidence and quantification}.
In Figure~\ref{wishlist3} we give
 a mode diagram for the actions, where we have used ovals to emphasise the points and lines where interventions occur. The filled in triangle (2-simplex) mode is responsible for monitoring the situation and issuing a warning to the offender if necessary. The point and line modes at the bottom are responsible for alerting the relevant services.

\begin{figure}[htbp]
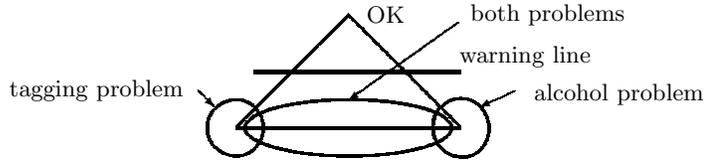

\begin{center}

\unitlength 0.5 mm


\caption{Diagram describing interventions in the alcohol-tagging problem}
\label{wishlist3}
\end{center}
\end{figure}

 We now consider a cover of the state space and its relation to the function $\underline\phi$ as discussed in Definition~\ref{realpu} and pictured in Figure~\ref{part4}.
We use the square state space $[0,1]^2$ and the cover given in 
 Figure~\ref{alctag3}. 
 There the first shaded region indicates an alcohol problem, the second a tagging problem, and the third the absence of a problem.
 The function $\phi$ is pictured in Figure~\ref{trisqu}.

\begin{figure}[htbp]
\begin{center}

\unitlength 0.7 mm
\begin{picture}(120,32)(-10,23)
\linethickness{0.3mm}
\put(10,50){\line(1,0){20}}
\put(10,30){\line(0,1){20}}
\put(30,30){\line(0,1){20}}
\put(10,30){\line(1,0){20}}
\linethickness{0.3mm}
\put(45,50){\line(1,0){20}}
\put(45,30){\line(0,1){20}}
\put(65,30){\line(0,1){20}}
\put(45,30){\line(1,0){20}}
\linethickness{0.3mm}
\put(80,50){\line(1,0){20}}
\put(80,30){\line(0,1){20}}
\put(100,30){\line(0,1){20}}
\put(80,30){\line(1,0){20}}
\linethickness{0.3mm}
\put(10,37.5){\line(1,0){20}}
\linethickness{0.3mm}
\put(52.5,30){\line(0,1){20}}
\linethickness{0.3mm}
\multiput(12.5,47.5)(0.12,-0.18){42}{\line(0,-1){0.18}}
\linethickness{0.3mm}
\multiput(17.5,47.5)(0.12,-0.18){42}{\line(0,-1){0.18}}
\linethickness{0.3mm}
\multiput(22.5,47.5)(0.12,-0.18){42}{\line(0,-1){0.18}}
\linethickness{0.3mm}
\multiput(55,47.5)(0.18,-0.12){42}{\line(1,0){0.18}}
\linethickness{0.3mm}
\multiput(55,42.5)(0.18,-0.12){42}{\line(1,0){0.18}}
\linethickness{0.3mm}
\multiput(55,37.5)(0.18,-0.12){42}{\line(1,0){0.18}}
\linethickness{0.3mm}
\multiput(82.5,36.25)(0.18,0.12){42}{\line(1,0){0.18}}
\linethickness{0.3mm}
\multiput(82.5,31.25)(0.18,0.12){42}{\line(1,0){0.18}}

\put(18.5,26){\makebox(0,0)[cc]{alcohol}}
\put(4,40){\makebox(0,0)[cc]{tag}}
\put(7.5,27){\makebox(0,0)[cc]{0}}
\put(31,26.5){\makebox(0,0)[cc]{1}}
\put(7,49){\makebox(0,0)[cc]{1}}

\linethickness{0.3mm}
\put(92.5,30){\line(0,1){12.5}}
\linethickness{0.3mm}
\put(80,42.5){\line(1,0){12.5}}
\end{picture}

\caption{A cover of $[0,1]^2$ for the alcohol-tagging problem}
\label{alctag3}
\end{center}
\end{figure}
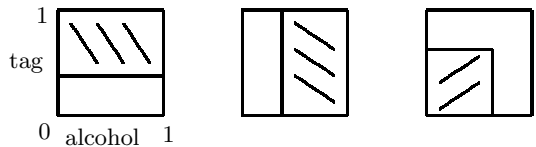

 \begin{figure}[htbp]
\begin{center}

 \unitlength 0.45 mm
\begin{picture}(187.5,50)(8,10)
\linethickness{0.8mm}
\multiput(113.12,32.5)(0.36,-0.12){167}{\line(1,0){0.36}}
\linethickness{0.3mm}
\multiput(153.12,62.5)(0.12,-0.3){167}{\line(0,-1){0.3}}
\linethickness{0.3mm}
\multiput(113.12,32.5)(0.16,0.12){250}{\line(1,0){0.16}}
\put(187.5,60){\makebox(0,0)[cc]{$({\mathrm{alcOK}},{\mathrm{tagOK}})$}}

%\put(213,42.5){\makebox(0,0)[cc]{$({\mathrm{alcProb}},{\mathrm{tagProb}})$}}

\put(203,13.12){\makebox(0,0)[cc]{$({\mathrm{alcProb}},{\mathrm{tagOK}})$}}
\put(197,28){\makebox(0,0)[cc]{warning line}}
\put(96,25){\makebox(0,0)[cc]{$({\mathrm{alcOK}},{\mathrm{tagProb}})$}}

\linethickness{0.3mm}
\put(10,20){\line(0,1){40}}
\linethickness{0.3mm}
\put(10,20){\line(1,0){40}}
\linethickness{0.8mm}
\put(50,20){\line(0,1){40}}
\linethickness{0.8mm}
\put(10,60){\line(1,0){40}}
\linethickness{0.3mm}
\put(10,60){\line(1,0){10}}
\put(10,50){\line(0,1){10}}
\put(20,50){\line(0,1){10}}
\put(10,50){\line(1,0){10}}
\linethickness{0.3mm}
\put(10,30){\line(1,0){10}}
\put(10,20){\line(0,1){10}}
\put(20,20){\line(0,1){10}}
\put(10,20){\line(1,0){10}}
\linethickness{0.3mm}
\put(40,30){\line(1,0){10}}
\put(40,20){\line(0,1){10}}
\put(50,20){\line(0,1){10}}
\put(40,20){\line(1,0){10}}
\linethickness{0.1mm}
\multiput(20,30)(0.5,0.12){250}{\line(1,0){0.5}}
\put(145,60){\vector(4,1){0.12}}
\linethickness{0.1mm}
\multiput(20,50)(0.72,-0.12){125}{\line(1,0){0.72}}
\put(110,35){\vector(4,-1){0.12}}
\linethickness{0.1mm}
\multiput(50,20)(1.83,-0.12){63}{\line(1,0){1.83}}
\put(165,12.5){\vector(1,-0){0.12}}
\put(27.5,16){\makebox(0,0)[cc]{alcohol}}

\put(4,42.5){\makebox(0,0)[cc]{tag}}

\put(10,16){\makebox(0,0)[cc]{0}}

\put(50,16){\makebox(0,0)[cc]{1}}

\put(7,60){\makebox(0,0)[cc]{1}}

\linethickness{0.8mm}
\put(10,50){\line(1,0){30}}
\linethickness{0.8mm}
\put(40,20){\line(0,1){30}}
%\put(40,22.5){\line(0,1){27.5}}
\linethickness{0.5mm}
\multiput(115,48)(0.36,-0.12){167}{\line(1,0){0.36}}
\end{picture}

\caption{The function $\underline{\phi}$ mapping $[0,1]^2$ to a 2-simplex for the alcohol-tagging problem}
\label{trisqu}
\end{center}
\end{figure}
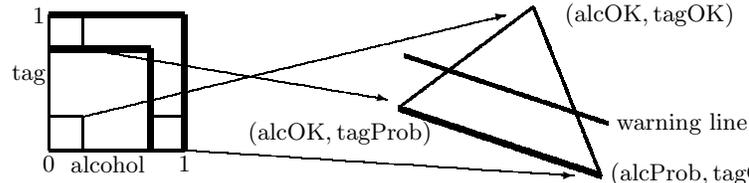

\begin{center}
  \begin{tabular}{@{} |c|c| @{}}
    \hline
    System &  Model \\ 
    \hline
    Algorithm & Automated system alerts professionals    \\  \hline
    Data & Subject's records, limited time memory \\ \hline
    Objectives & Prevent incident, promote treatment\\ \hline
    Oracles & Electronic tag, blood alcohol device \\ \hline
    Visualisation of  belief & Supporting decision in court\\ 
    \hline
  \end{tabular}
\end{center}

\subsection{Hospital admission}\label{admission}\label{man5}

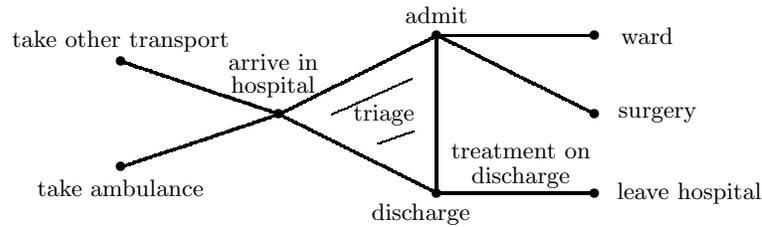
\begin{figure}[htbp]
\begin{center}

\unitlength 0.7 mm
\begin{picture}(105,50)(0,25)
\linethickness{0.3mm}
\multiput(10,60)(0.36,-0.12){83}{\line(1,0){0.36}}
\linethickness{0.3mm}
\multiput(10,40)(0.36,0.12){83}{\line(1,0){0.36}}
\linethickness{0.3mm}
\multiput(40,50)(0.24,0.12){125}{\line(1,0){0.24}}
\linethickness{0.3mm}
\put(70,35){\line(0,1){30}}
\linethickness{0.3mm}
\multiput(40,50)(0.24,-0.12){125}{\line(1,0){0.24}}
\linethickness{0.3mm}
\put(70,65){\line(1,0){30}}
\linethickness{0.3mm}
\put(70,35){\line(1,0){30}}
\linethickness{0.3mm}
\multiput(70,65)(0.24,-0.12){125}{\line(1,0){0.24}}
\put(110,65){\makebox(0,0)[cc]{ward}}

\put(112,50){\makebox(0,0)[cc]{surgery}}

\put(118,35){\makebox(0,0)[cc]{leave hospital}}

\put(70,69){\makebox(0,0)[cc]{admit}}

\put(67,31){\makebox(0,0)[cc]{discharge}}

\put(86,43){\makebox(0,0)[cc]{treatment on}}
\put(86,38){\makebox(0,0)[cc]{discharge}}

\put(60,50){\makebox(0,0)[cc]{triage}}

\put(39,60){\makebox(0,0)[cc]{arrive in}}
\put(39,55.2){\makebox(0,0)[cc]{hospital}}

\put(10,36){\makebox(0,0)[cc]{take ambulance}}

\put(10,64){\makebox(0,0)[cc]{take other transport}}

\linethickness{0.1mm}
\multiput(50,50)(0.27,0.12){57}{\line(1,0){0.27}}
\linethickness{0.1mm}
\multiput(58.75,44.38)(0.33,0.12){21}{\line(1,0){0.33}}

\put(10,40){\makebox(0,0)[cc]{$\bullet$}}
\put(10,60){\makebox(0,0)[cc]{$\bullet$}}
\put(40,50){\makebox(0,0)[cc]{$\bullet$}}
\put(70,65){\makebox(0,0)[cc]{$\bullet$}}
\put(70,35){\makebox(0,0)[cc]{$\bullet$}}
\put(100,35){\makebox(0,0)[cc]{$\bullet$}}
\put(100,50){\makebox(0,0)[cc]{$\bullet$}}
\put(100,65){\makebox(0,0)[cc]{$\bullet$}}

\end{picture}

\setlength{\belowcaptionskip}{-15pt}
\caption{Visualising unplanned admissions to hospital}
\label{modpic3}
\end{center}
\end{figure}

\noindent \textit{Scenario and its visualisation}.  In Figure~\ref{modpic3} we consider a person arriving at a hospital and entering an initial assessment process or triage.
 In the mode `in ambulance', the line between `take ambulance' and `arrive', the patient will be in the care of paramedics who may inform the hospital of an imminent arrival. 
 
 On arrival the patient enters the `triage' mode, the filled in triangle with vertices `arrive', `admit' and `discharge'. The purpose of this mode is to sort patients into those who have to stay in hospital and those who can go home. As evidence is collected the state of the patient will move along the triangle until it reaches one of the vertices `admit' or `discharge'. If the patient is known to be critical (e.g.\ by a report from the paramedics) then they may proceed along the line `emergency admission' connecting `arrive' and `admit'. If the hospital is closed to admissions then the `compulsory discharge' mode connecting `arrive' and `discharge' will be used.
 
  If the decision is made to discharge the patient may still receive treatment before leaving, indicated by the line leading to `leave'. Similarly being sent to surgery or a ward requires preparation. 

We will concentrate on the triage process itself, the filled in triangle in Figure~\ref{modpic3}. In contrast to the simple state space in the last example, we have $\mathbb{S}$ consisting of the possible states of health of a person, an enormously complicated space. Fortunately, we only have to make a decision to admit or discharge, and it is that path of likelihood of decisions as a function of time during the triage process that is plotted on the diagram.
  
  \medskip
\noindent\textit{Evidence and quantification}. 
To gather evidence we use a query tree. At each node in the tree a question is asked or an observation is made, and the next node is chosen depending on the result. Each note will have its own score for the three coordinates
\[
\phi(node)=
\big( \mathrm{begin\ triage}\ ,\ \mathrm{discharge}\ ,\ \mathrm{admit}\big)\ .
\]
This is a triple of positive numbers adding to one, and specifies a point in the triangle. 
For the later vertices, we might also get a list of possible conditions which will be added to the patients notes. The list of coordinates gives the function $\phi$, and this function is updated as the examination proceeds. Early in the process the `begin triage' value will be close to 1, and it will decrease as the examination continues.

\medskip
\noindent\textit{Thresholds and mode transitions}.  
The threshold values for decisions may depend on other factors. For example, if the hospital is not busy a patient might be admitted as soon as the `admit' coordinate reaches $0.8$, whereas on a very busy day it might have to reach $0.95$.

\begin{center}
  \begin{tabular}{@{} |c|c| @{}}
    \hline
    System &  Model \\ 
    \hline
    Algorithm & Formulae as helpers only, staff make decisions    \\  \hline
    Data & Medical records,\\ \hline
    Objectives & Triage to a set reliability in a set average time\\ \hline
    Oracles & Observations by medical staff, tests, patients narrative \\ \hline
    Visualisation of  belief & Expert system, medical staff have final say\\ 
    \hline
  \end{tabular}
\end{center}

 \begin{remark} \label{hosp}
 This example can be seen from several different points of view. The account above is from the point of view of the hospital. There are other agents at work, most obviously the machines; the ambulance and the medical equipment; and the people, the patient and the medical staff. 
 All of these have their own state spaces and their own reasons for acting and communicating as they do. (E.g.\ a blood pressure monitor may not have been properly maintained.)
\end{remark}

%%%%%%%%%%%%%%%%%%%%%%%%%%%%%%%%%%%%%%%%%%%%%%%%%%%%%
%%%%%%%%%%%%%%%%%%%%%%%%%%%%%%%%%%%%%%%%%%%%%%%%%%%%%

%%%%%%%%%%%%
\subsection{A judicial process}

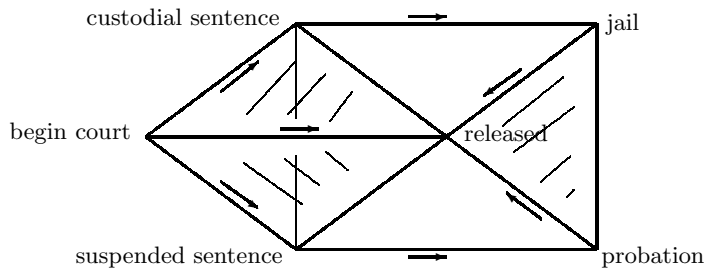
\begin{figure}[htbp]
\begin{center}

\unitlength 0.5 mm
\begin{picture}(125,80)(0,-1)
\linethickness{0.3mm}
\put(0,40){\line(1,0){80}}
\linethickness{0.3mm}
\multiput(40,70)(0.16,-0.12){250}{\line(1,0){0.16}}
\linethickness{0.3mm}
\multiput(0,40)(0.16,0.12){250}{\line(1,0){0.16}}
\linethickness{0.3mm}
\multiput(0,40)(0.16,-0.12){250}{\line(1,0){0.16}}
\linethickness{0.3mm}
\multiput(40,10)(0.16,0.12){250}{\line(1,0){0.16}}
\linethickness{0.3mm}
\multiput(80,40)(0.16,0.12){250}{\line(1,0){0.16}}
\linethickness{0.3mm}
\put(120,10){\line(0,1){60}}
\linethickness{0.3mm}
\multiput(80,40)(0.16,-0.12){250}{\line(1,0){0.16}}
\put(0,40){\makebox(0,0)[cc]{}}

\linethickness{0.3mm}
\put(40,70){\line(1,0){80}}
\linethickness{0.3mm}
\put(40,10){\line(1,0){80}}
\put(96,41){\makebox(0,0)[cc]{released}}

\put(127,70){\makebox(0,0)[cc]{jail}}

\put(10,72){\makebox(0,0)[cc]{custodial sentence}}

\put(9,8){\makebox(0,0)[cc]{suspended sentence}}

\put(-20,41){\makebox(0,0)[cc]{begin court}}

\put(135,8){\makebox(0,0)[cc]{probation}}

\linethickness{0.15mm}
\put(40,45){\line(0,1){25}}
\linethickness{0.15mm}
\put(40,10){\line(0,1){25}}
\linethickness{0.3mm}
\multiput(20,52)(0.15,0.12){67}{\line(1,0){0.15}}
\put(30,60){\vector(4,3){0.12}}
\linethickness{0.3mm}
\put(36,42){\line(1,0){10}}
\put(46,42){\vector(1,0){0.12}}
\linethickness{0.3mm}
\multiput(20,28)(0.17,-0.12){58}{\line(1,0){0.17}}
\put(30,21){\vector(3,-2){0.12}}
\linethickness{0.3mm}
\put(70,72){\line(1,0){10}}
\put(80,72){\vector(1,0){0.12}}
\linethickness{0.3mm}
\multiput(90,51)(0.17,0.12){58}{\line(1,0){0.17}}
\put(90,51){\vector(-3,-2){0.12}}
\linethickness{0.3mm}
\put(70,8){\line(1,0){10}}
\put(80,8){\vector(1,0){0.12}}
\linethickness{0.3mm}
\multiput(96,25)(0.16,-0.12){58}{\line(1,0){0.16}}
\put(96,25){\vector(-4,3){0.12}}
%\linethickness{0.3mm}
%\put(122,34){\line(0,1){10}}
%\put(122,34){\vector(0,-1){0.12}}
\linethickness{0.1mm}
\multiput(95,43)(0.15,0.12){108}{\line(1,0){0.15}}
\linethickness{0.1mm}
\multiput(98,36)(0.16,0.12){100}{\line(1,0){0.16}}
\linethickness{0.1mm}
\multiput(105,28)(0.14,0.12){58}{\line(1,0){0.14}}
\linethickness{0.1mm}
\multiput(112,24)(0.12,0.12){17}{\line(1,0){0.12}}
\linethickness{0.1mm}
\multiput(27,46)(0.12,0.13){108}{\line(0,1){0.13}}
\linethickness{0.1mm}
\multiput(38,46)(0.12,0.13){83}{\line(0,1){0.13}}
\linethickness{0.1mm}
\multiput(49,44)(0.12,0.16){50}{\line(0,1){0.16}}
\linethickness{0.1mm}
\multiput(26,33)(0.17,-0.12){92}{\line(1,0){0.17}}
\linethickness{0.1mm}
\multiput(37,34)(0.16,-0.12){58}{\line(1,0){0.16}}
\linethickness{0.1mm}
\multiput(48,36)(0.14,-0.12){42}{\line(1,0){0.14}}
\end{picture}

\setlength{\belowcaptionskip}{-15pt}
\caption{A simplified judicial process}
\label{judicial}
\end{center}
\end{figure}

\noindent \textit{Scenario and its visualisation}.  
Consider someone who is accused of a crime appearing in court. In Figure~\ref{judicial} they enter on the leftmost vertex.
In our simplified picture the court can take one of three options: innocent (released), guilty (custodial sentence-- go to jail) or guilty (suspended or deferred sentence -- allowed to leave under the supervision of the probation service). This can be represented by a 3-simplex (solid tetrahedron) with the weights of the vertices being the belief in the court in the three possible decisions and a weight indicating whether the court believes that it is time to reach a decision, i.e.\ whether the evidence has been considered. In terms of the previous subsets, the 3-simplex corresponds to the vertices
\[
\{\mathrm{begin\ court,custodial,release,suspended}\}\in\mathcal{C}
\]
 Initially we start at the begin court vertex, as no evidence has been presented, and (ideally) the court should not have a prior belief about the verdict. As time passes and the evidence is presented the beliefs change and we flow away from the court vertex, towards the 2-simplex giving the judgements (this is indicated by the arrows giving a control direction to the system). 

The system then enters the 2-simplex on the right. It would be possible to move 
from jail to probation (time served and good behaviour) or the other way for breaking conditions. The flow lines here indicate the progress towards the eventual end of the sentence (possibly reduced for good behaviour).

  \medskip
\noindent\textit{Evidence and quantification}. 
The court process (the solid tetrahedron in Fig.~\ref{judicial}) has some similarities to the medical triage we discussed earlier, in that it depends on expert opinions. Instead, we look at a more straightforward part of the proceedings.
We choose to look at the triangle in Figure~\ref{judicial} formed by the vertices released, jail and probation in  detail. 
The state space $\mathbb{S}$ will be the behaviour of the offender up to the current time $t$. (We measure time $t\in [0,1]$ as a fraction of the original sentence.) For simplicity we summarise the behaviour of the offender into a single good behaviour function $g(t)\in [0,1]$. We take $g(t)=0$ to be bad at time $t$, and $g(t)=1$ to be good. The function $g(t)$ could be given by a start value (say $g(0)=1$) at the beginning of the sentence. Subsequently every recorded incident of bad behaviour subtracts an amount from $g(t)$, but this deduction decays with time. 

 We now consider the 2-simplex in Figure~\ref{jpttrt} and the visualisation of the trajectory $\phi$ of the offender  in the system as a function of time. 
 At any time the offender is in one of the three states Released ($r$), Jail ($j$) or Probation ($p$). we will consider transitions out of these states, and fill in the $(r,j,p)$ values afterwards. We only specify these values on the boundary of the given domains, and fill in the other values by interpolation. (We could take a much more precise approach, but it would only obscure what is going on.)

\medskip
\noindent\textit{Thresholds and mode transitions}.  
As regards transferring between the three possibilities, we do not have a single threshold value. This might result in someone being moved between jail and probation every single day if they were near that value. Instead we assume that we have overlapping stable regions.
There are many ways in which we might choose to model this, but we shall choose catastrophe theory. This was invented to give a theory of continuous changes giving rise to discontinuous effects (see \cite{arnCat} for theory and \cite{ZeeSciAm} for applications).

 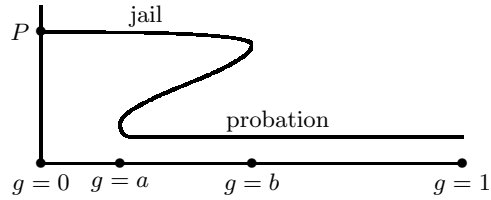
\begin{figure}[htbp]
\begin{center}
\unitlength 0.7 mm
\begin{picture}(90,42)(0,30)
\linethickness{0.3mm}
\qbezier(10,65)(30.88,65.01)(40.5,64.41)
\qbezier(40.5,64.41)(50.12,63.8)(50,62.5)
\qbezier(50,62.5)(50.04,61.21)(47.03,59.41)
\qbezier(47.03,59.41)(44.02,57.6)(37.5,55)
\qbezier(37.5,55)(30.98,52.4)(27.97,50.59)
\qbezier(27.97,50.59)(24.96,48.79)(25,47.5)
\qbezier(25,47.5)(24.99,46.2)(25.59,45.59)
\qbezier(25.59,45.59)(26.2,44.99)(27.5,45)
\qbezier(27.5,45)(28.71,45)(36.53,45)
\qbezier(36.53,45)(44.35,45)(60,45)
\qbezier(60,45)(75.66,45)(82.88,45)
\qbezier(82.88,45)(90.09,45)(90,45)
\linethickness{0.3mm}
\put(10,40){\line(0,1){30}}
\linethickness{0.3mm}
\put(10,40){\line(1,0){80}}
\put(55,48){\makebox(0,0)[cc]{probation}}

\put(30,68){\makebox(0,0)[cc]{jail}}

\put(90,40){\makebox(0,0)[cc]{$\bullet$}}

\put(10,40){\makebox(0,0)[cc]{$\bullet$}}
\put(10,65){\makebox(0,0)[cc]{$\bullet$}}

\put(50,40){\makebox(0,0)[cc]{$\bullet$}}

\put(25,40){\makebox(0,0)[cc]{$\bullet$}}

\put(10,36){\makebox(0,0)[cc]{$g=0$}}
\put(25,36){\makebox(0,0)[cc]{$g=a$}}
\put(50,36){\makebox(0,0)[cc]{$g=b$}}
\put(90,36){\makebox(0,0)[cc]{$g=1$}}

\put(6,65){\makebox(0,0)[cc]{$P$}}

\end{picture}
\setlength{\belowcaptionskip}{-15pt}
\caption{Moving from jail to and from probation}
\label{caty}
\end{center}
\end{figure}

In Fig.~\ref{caty} we model the transition from jail to probation, assuming that the transitions are time independent so the result only depends on good behaviour $g$. Beginning at point $P$, with $g=0$ in jail, a character decides to reform and increases their $g$ value. They move to the right along the curve until they reach $g=b$, at which point they `fall off' the jail surface onto the probation surface. If they behave more badly they will not be moved back to jail until they move left along the probation curve to the point $g=a$, at which point the possibility of remaining on probation disappears and they end back in jail. 
Seen from above, we would have a stable jail region from $g=0$ to $g=b$ and a stable jail region from $g=a$ to $g=1$. If a person is in one of these regions, they stay in it until they cross the boundary. 

Now we consider how to visualise this in a simplicial complex, in this simplified case an interval with end points `jail' and `probation'. 
In Fig.~\ref{caty} for $g<a$ there is no possibility other than jail, and likewise for $g>b$ we have probation. The interesting intermediate case is when $a\le g\le b$, when the person has a possibility of changing status, depending on how they behave. We give a suitable function from the state space (in this case just the $g\in [0,1]$ values) to the simplicial complex in Fig.~\ref{caty2}. 

 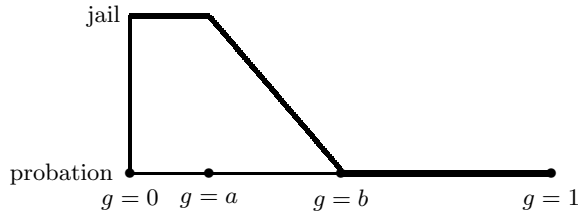
\begin{figure}[htbp]
\begin{center}
\unitlength 0.7 mm
\begin{picture}(90,42)(0,30)
\linethickness{0.3mm}
\put(10,40){\line(0,1){30}}
\linethickness{0.3mm}
\put(10,40){\line(1,0){80}}
\put(-3,40){\makebox(0,0)[cc]{probation}}

\put(5,70){\makebox(0,0)[cc]{jail}}

\put(90,40){\makebox(0,0)[cc]{$\bullet$}}

\put(10,40){\makebox(0,0)[cc]{$\bullet$}}

\put(50,40){\makebox(0,0)[cc]{$\bullet$}}

\put(25,40){\makebox(0,0)[cc]{$\bullet$}}

\put(10,35){\makebox(0,0)[cc]{$g=0$}}

\put(90,35){\makebox(0,0)[cc]{$g=1$}}

\linethickness{0.7mm}
\put(10,70){\line(1,0){15}}
\linethickness{0.7mm}
\multiput(25,70)(0.12,-0.14){208}{\line(0,-1){0.14}}
\linethickness{0.7mm}
\put(50,40){\line(1,0){40}}
\put(25,35){\makebox(0,0)[cc]{$g=a$}}

\put(50,35){\makebox(0,0)[cc]{$g=b$}}

\end{picture}
\setlength{\belowcaptionskip}{-15pt}
\caption{The map into the 1-simplex for jail and probation}
\label{caty2}
\end{center}
\end{figure}

Now we consider adding time $t\in [0,1]$ as a fraction of the original sentence to the state space, and being released to the simplicial complex. 

 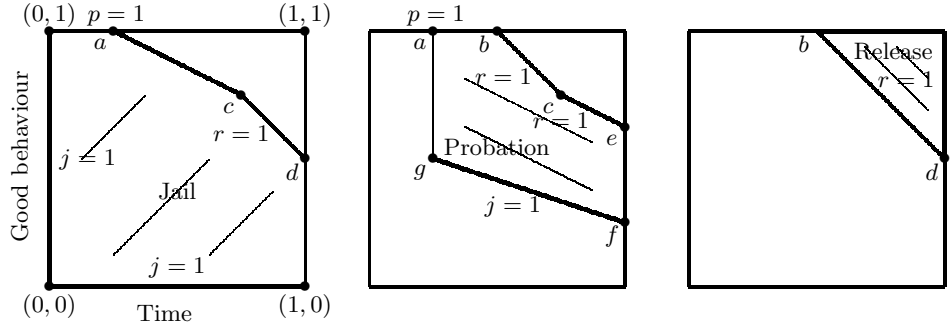
\begin{figure}[htbp]
\begin{center}
\unitlength 0.85 mm
\begin{picture}(150,50)(8,4)
\linethickness{0.5mm}
\put(10,10){\line(0,1){40}}
\linethickness{0.3mm}
\put(10,50){\line(1,0){40}}
\linethickness{0.3mm}
\put(50,10){\line(0,1){40}}
\linethickness{0.5mm}
\put(10,10){\line(1,0){40}}
\linethickness{0.3mm}
\put(60,10){\line(0,1){40}}
\linethickness{0.3mm}
\put(60,10){\line(1,0){40}}
\linethickness{0.3mm}
\put(100,10){\line(0,1){40}}
\linethickness{0.3mm}
\put(60,50){\line(1,0){40}}
\linethickness{0.3mm}
\put(110,50){\line(1,0){40}}
\linethickness{0.3mm}
\put(150,10){\line(0,1){40}}
\linethickness{0.3mm}
\put(110,10){\line(1,0){40}}
\linethickness{0.3mm}
\put(110,10){\line(0,1){40}}
\linethickness{0.5mm}
\multiput(70,30)(0.36,-0.12){83}{\line(1,0){0.36}}
\linethickness{0.1mm}
\put(70,30){\line(0,1){20}}
\linethickness{0.5mm}
\multiput(20,50)(0.24,-0.12){83}{\line(1,0){0.24}}
\linethickness{0.5mm}
\multiput(40,40)(0.12,-0.12){83}{\line(1,0){0.12}}
\linethickness{0.5mm}
\multiput(130,50)(0.12,-0.12){167}{\line(1,0){0.12}}
\linethickness{0.5mm}
\multiput(80,50)(0.12,-0.12){83}{\line(1,0){0.12}}
\linethickness{0.5mm}
\multiput(90,40)(0.24,-0.12){42}{\line(1,0){0.24}}
\put(70,50){\makebox(0,0)[cc]{$\bullet$}}

\linethickness{0.05mm}
\multiput(15,30)(0.12,0.12){83}{\line(1,0){0.12}}
\linethickness{0.05mm}
\multiput(20,15)(0.12,0.12){125}{\line(1,0){0.12}}
\linethickness{0.05mm}
\multiput(35,15)(0.12,0.12){83}{\line(1,0){0.12}}
\linethickness{0.05mm}
\multiput(75,42.5)(0.24,-0.12){83}{\line(1,0){0.24}}
\linethickness{0.1mm}
\multiput(75,35)(0.24,-0.12){83}{\line(1,0){0.24}}
\linethickness{0.05mm}
\multiput(137.5,47.5)(0.12,-0.12){83}{\line(1,0){0.12}}
\linethickness{0.05mm}
\multiput(142.5,47.5)(0.12,-0.12){42}{\line(1,0){0.12}}
\put(30,25){\makebox(0,0)[cc]{Jail}}

\put(80,32){\makebox(0,0)[cc]{Probation}}

\put(142,47){\makebox(0,0)[cc]{Release}}

\linethickness{0.5mm}
\put(130,50){\line(1,0){20}}
\linethickness{0.5mm}
\put(150,30){\line(0,1){20}}
\put(80,50){\makebox(0,0)[cc]{$\bullet$}}

\put(70,30){\makebox(0,0)[cc]{$\bullet$}}

\put(16,30){\makebox(0,0)[cc]{$j=1$}}

\put(30,10){\makebox(0,0)[cc]{}}

\put(30,13){\makebox(0,0)[cc]{$j=1$}}

\put(20.5,52.5){\makebox(0,0)[cc]{$p=1$}}
\put(70.5,52.5){\makebox(0,0)[cc]{$p=1$}}

\put(40,34){\makebox(0,0)[cc]{$r=1$}}

\put(82.5,22.5){\makebox(0,0)[cc]{$j=1$}}

\put(90,36){\makebox(0,0)[cc]{$r=1$}}
\put(81,43){\makebox(0,0)[cc]{$r=1$}}

\put(144,42.5){\makebox(0,0)[cc]{$r=1$}}

\put(20,50){\makebox(0,0)[cc]{$\bullet$}}

\put(18,48){\makebox(0,0)[cc]{$a$}}
\put(68,48){\makebox(0,0)[cc]{$a$}}
\put(78,48){\makebox(0,0)[cc]{$b$}}
\put(128,48){\makebox(0,0)[cc]{$b$}}
\put(38,38){\makebox(0,0)[cc]{$c$}}
\put(88,38){\makebox(0,0)[cc]{$c$}}
\put(90,40){\makebox(0,0)[cc]{$\bullet$}}
\put(40,40){\makebox(0,0)[cc]{$\bullet$}}
\put(48,28){\makebox(0,0)[cc]{$d$}}
\put(148,28){\makebox(0,0)[cc]{$d$}}
\put(50,30){\makebox(0,0)[cc]{$\bullet$}}
\put(150,30){\makebox(0,0)[cc]{$\bullet$}}

\put(100,35){\makebox(0,0)[cc]{$\bullet$}}
\put(98,33){\makebox(0,0)[cc]{$e$}}
\put(100,20){\makebox(0,0)[cc]{$\bullet$}}
\put(98,18){\makebox(0,0)[cc]{$f$}}
\put(68,28){\makebox(0,0)[cc]{$g$}}

\put(10,7){\makebox(0,0)[cc]{$(0,0)$}}
\put(50,7){\makebox(0,0)[cc]{$(1,0)$}}
\put(10,52.5){\makebox(0,0)[cc]{$(0,1)$}}
\put(50,52.5){\makebox(0,0)[cc]{$(1,1)$}}
\put(10,10){\makebox(0,0)[cc]{$\bullet$}}
\put(10,50){\makebox(0,0)[cc]{$\bullet$}}
\put(50,10){\makebox(0,0)[cc]{$\bullet$}}
\put(50,50){\makebox(0,0)[cc]{$\bullet$}}

\put(4,17){\rotatebox{90}{Good behaviour}}
\put(28,6){\makebox(0,0)[cc]{Time}}

\end{picture}

\setlength{\belowcaptionskip}{-15pt}
\caption{Domains for jail, probation and early release}
\label{jprt}
\end{center}
\end{figure}

In Figure~\ref{jprt} we the stable domains for jail, probation and early release and the corresponding maps into the triangle with vertices $j,p,r$ respectively. As previously mentioned, a convict will have a path through the square of good behaviour as a function of time (as a proportion of sentence). \newline
By \textit{stable domains} we mean that if a convict is in the shaded domain `Jail' in the left square of Figure~\ref{jprt}, then they will remain in jail unless their path leaves the shaded region. If their path leaves the Jail region on the interval $ac$ then they are transferred to probation. If it leaves on the interval $cd$ then they are given early release. \newline
Correspondingly, a convict in the domain `Probation' can exit to early release along the lines $bce$ and to jail along
$gf$. (They cannot exit by crossing $ag$ unless they have a time machine.)

 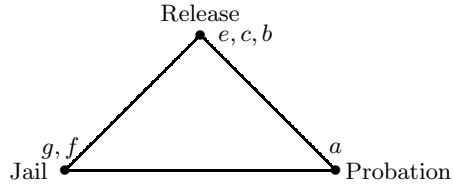
\begin{figure}[htbp]
\begin{center}
\unitlength 0.6 mm
\begin{picture}(85,42)(0,15)
\linethickness{0.3mm}
\multiput(20,20)(0.12,0.12){250}{\line(1,0){0.12}}
\linethickness{0.3mm}
\multiput(50,50)(0.12,-0.12){250}{\line(1,0){0.12}}
\linethickness{0.3mm}
\put(20,20){\line(1,0){60}}
\put(50,55){\makebox(0,0)[cc]{Release}}

\put(94,20){\makebox(0,0)[cc]{Probation}}

\put(12,20){\makebox(0,0)[cc]{Jail}}

\put(50,50){\makebox(0,0)[cc]{$\bullet$}}

\put(80,20){\makebox(0,0)[cc]{$\bullet$}}

\put(20,20){\makebox(0,0)[cc]{$\bullet$}}

\put(19,25){\makebox(0,0)[cc]{$g,f$}}

\put(60,50){\makebox(0,0)[cc]{$e,c,b$}}

\put(80,25){\makebox(0,0)[cc]{$a$}}

\end{picture}
\setlength{\belowcaptionskip}{-15pt}
\caption{Map to the 2-simplex for the Probation domain}
\label{jpttrt}
\end{center}
\end{figure}

To show the map to the simplex for Figure~\ref{jprt} we take the most complicated case of the probation domain shown in
Figure~\ref{jpttrt}. The map is interpolated between the values on the boundary to give values on the whole probation domain in a continuous manner.

\begin{center}
  \begin{tabular}{@{} |c|c| @{}}
    \hline
    System &  Model \\ 
    \hline
    Algorithm & Given by domains in Fig.~\ref{jprt}.  \\  \hline
    Data & Time $t$, good behaviour function $g(t)$. \\ \hline
    Objectives & Decide on release or parole.  \\ \hline
    Oracles & Reports on behaviour from warders, police, probation offlcers. \\ \hline
    Visualisation of  belief & Show how close the offender is to a transition. \\ 
    \hline
  \end{tabular}
\end{center}

\section{Mathematical Tools 2: Belief and visualisation} \label{belvis4}
%
% \begin{quotation}
% \noindent
%Why, sometimes I've believed as many as six impossible things before breakfast.
%
%\noindent \rightline{\textit{Lewis Carroll, The white queen in Through the looking-glass}} 
%\end{quotation}
%

In complex systems beliefs change. When a plane approaches within, say, 100 miles of its target airport the pilot will believe that it should prepare to land. If it is then found to be running out of fuel then the belief in landing will rise to absolute necessity, even overriding bad weather warnings. Thus, we emphasise that we are using beliefs as dynamic quantities, they are not merely assessing the likelihood of static conditions. 

We shall use the Dempster-Shafer theory of belief functions \cite{DemSchBel} to quantify belief on a scale of 0 (no evidence) to 1 (convincing evidence). There is a large literature on this subject, but it shares the idea of a measure for belief which can be updated on the presentation of additional evidence. Here we are not concerned with any particular mechanism of how this update is performed. 

%%%%%%%%%%%%%%%%%%%%%%%%%%%%%%%%%%%%%%%%%%%%%%%%%%%%%
%%%%%%%%%%%%%%%%%%%%%%%%%%%%%%%%%%%%%%%%%%%%%%%%%%%%%

\subsection {Dempster-Shafer belief} \label{bel99}

 It is important to distinguish between belief in something and a probability that that thing occurs.
 On being told by a completely reliable source that a person has a car which has a single colour, we could form two statements, 
\begin{center}
$B=$`the car is blue' and $NB=$`the car is not blue'.
\end{center}

The combination $B$ or $NB$ we would consider to have a belief of 1, but in the absence of any other information we would have no \textit{evidence base} to assign a non-zero belief value to the separate statement $B$ or to the statement $NB$. We could randomly assign belief values to each which added up to 1, but this would only undermine the idea that \textit{belief ought be assigned according to evidence}. Sensibly our belief $\mathrm{Bel}(\{B,NB\})$ in the statement ($B$ or $NB$) is strictly greater than the sum of $\mathrm{Bel}(\{B\})$ and $\mathrm{Bel}(\{NB\})$. Such a conclusion means that we are not dealing with a probability distribution.

 \begin{definition}\label{beldef}
  We consider beliefs in a finite set $X$ of statements.
A {\em generalised belief function} on $X$ is a function 
$$\mathrm{Bel}:P(X)\to [0,1],$$ 
where $P(X)$ is the set of subsets of $X$. Furthermore, for the empty set $\mathrm{Bel}(\emptyset)=0$ and
the sets satisfy the `super-additivity' property
\begin{eqnarray} \label{superadd}
\mathrm{Bel}(Y\cup Z)+\mathrm{Bel}(Y\cap Z)\ge \mathrm{Bel}(Y)+\mathrm{Bel}(Z)\quad\mathrm{for\ all}\ Y,Z\subset X.
\end{eqnarray}
 We denote by $\mathcal{B}(X)$ the set of belief functions on $X$. 
\end{definition}
 
We shall say that a belief function is {\em normalised} if $\mathrm{Bel}(X)=1$. It is important to note that we do \textit{not} impose this normalisation condition -- hence our word \textit{generalised} above. 
Thus, we reserve the right to believe that the set of presented alternatives $X$ may be incomplete.
The `super-additivity' property is simply justified in that, in addition to those people who believe in $Y$ and those who believe in $Z$, there may be people who just believe that at least one is true. 

Dual to the belief function is this:

\begin{definition}\label{plaus}
The {\em plausibility function}
$$\mathrm{Pla}:P(X)\to [0,1]$$
is defined by 
$$\mathrm{Pla}(Y)=\mathrm{Bel}(X)-\mathrm{Bel}(X\setminus Y),$$
where $X\setminus Y$  is the complement of $Y$ in $X$. 
\end{definition}

The plausibility of a subset of statements may be taken as the \textit{extent} to which people think that it might possibly be true. 

\begin{proposition}
$\mathrm{Bel}(Y)\le \mathrm{Pla}(Y)$ for all $Y\subset X.$
\end{proposition}

\begin{proof}
Note that
$$\mathrm{Bel}(Y)\le \mathrm{Bel}(Y\cup (X\setminus Y))-\mathrm{Bel}(X\setminus Y)=\mathrm{Bel}(X)-\mathrm{Bel}(X\setminus Y) = \mathrm{Pla}(Y).$$ 
\qed
\end{proof}

Also, we can deduce that the plausibility function is sub-additive, i.e.,\
\[
\mathrm{Pla}(Y\cup Z)+\mathrm{Pla}(Y\cap Z)\le \mathrm{Pla}(Y)+\mathrm{Pla}(Z)\quad\mathrm{for\ all}\ Y,Z\subset X.\
\]

%%%%%%%%%%%%%%%%%%%%%%%%%%%%%%%%%%%%%%%%%%%
%%%%%%%%%%%%%%%%%%%%%%%%%%%%%%%%%%%%%%%%%%%

\subsection{Visualisation of a belief function} \label{visbelop}
In principle, belief functions give a large amount of data which is difficult to understand -- for a set $X$ of size $n$ we have $2^n$ values.\footnote{For a probability measure we would only need the $n$ point measures by additivity, but belief functions contain much more information.} 
 However, as long as we are clear about what information we actually want, it is possible to visualise the results in a much lower dimension. 
 
 So, in practice, what do we need to know about beliefs? 
 
 \smallskip
 \noindent (1) We need to know when to add to our list of things $X$ to believe -- when our existing
 tool collection is \textit{incomplete}, as described in subsection~\ref{actions}.
 
 \noindent (2) We need to subtract things from $X$ which we no longer think are necessary, in accordance with \textit{Occam's razor}, as described in subsection~\ref{actions}.
 
 \smallskip
 Now (1) has a simple answer, even though in practice that answer might lead to many more questions.

\begin{definition} \label{firrel}
Let $\mathcal{B}(X)$ be the set of belief functions on $X$.
The first piece of the visualisation of a belief function $\mathrm{Bel}\in \mathcal{B}(X)$ is the belief in the whole of $X$, that is $\mathrm{Bel}(X)\in [0,1]$. 
\end{definition}

For (2), how can we visualise a belief in subsets of $X$ so that it is obvious which objects in $X$ are no longer needed? Answer: we visualise the belief as a point  in the simplex $\Delta_X$ spanned by $X$, as given  in 
Proposition~\ref{simpX}.

\begin{definition}\label{secrel}
The second piece of the visualisation of a belief function $\mathrm{Bel}\in \mathcal{B}(X)$ is the point in the simplex 
\[
\frac{\sum_{x\in X} \mathrm{Pla}(\{x\})\,e_x}{\sum_{x\in X} \mathrm{Pla}(\{x\}))} \in  \Delta_X\ .
\]
\end{definition}

The first thing to note is that the behaviour of this fraction is, shall we say, erratic if the denominator is small. In our case, by  sub-additivity, the denominator is 
$$\sum_{x\in X} \mathrm{Pla}(\{x\})\ge \mathrm{Pla}(X)= \mathrm{Bel}(X).$$ 
Thus, if $\mathrm{Bel}(X)\in[0,1]$ is close to 1, it is reasonable to take the fraction. If it is close to zero we may have problems, but in that case our belief in the set $X$ is so small that all conclusions will be considered unreliable. 

The second thing to note is that this function in Definition~\ref{secrel} really does solve requirement (2). 
To make this precise we need the idea of a face of a simplex. 

\begin{definition}
For $Y\subset X$ the {\em face}
$\Delta_Y\subset\Delta_X$ of $\Delta_X$  is
\[
\Delta_Y =\Big\{ \sum_{y\in Y} \lambda_y\, e_y : \lambda_y\ge 0,\ \sum _{y\in Y} \lambda_y=1\Big\}   \subset \Delta_X\ .
\]
\end{definition}

\begin{proposition} \label{vischeck}
For a belief function on $X$ with $ \mathrm{Bel}(X)\neq 0$ and $Y\subset X$
\[
 \mathrm{Bel}(Y)= \mathrm{Bel}(X) \ \mathrm{if, and \ only,  if}\ 
\frac{\sum_{x\in X} \mathrm{Pla}(\{x\})\,e_x}{\sum_{x\in X} \mathrm{Pla}(\{x\}))} \in  \Delta_Y \ .
\]
\end{proposition}

\noindent\textbf{Proof:}\quad First, we assume the RHS. Then, for $x\in X\setminus Y$ we have
$\mathrm{Pla}(\{x\})=0$, and, by sub-additivity, $\mathrm{Pla}(X\setminus Y)=0$;  by definition of 
$\mathrm{Pla}$, we get $ \mathrm{Bel}(Y)= \mathrm{Bel}(X)$.

Second, we assume the LHS. Then for $x\in X\setminus Y$ we have $ \mathrm{Bel}(X)\ge  \mathrm{Bel}(X\setminus \{x\}) \ge  \mathrm{Bel}(Y)= \mathrm{Bel}(X) $ so 
$ \mathrm{Bel}(X\setminus \{x\})= \mathrm{Bel}(X)$. We deduce $\mathrm{Pla}(\{x\})=0$ and so
we get the RHS. \hfill $\square$

\medskip

\begin{definition} \label{visdef}
For a belief function on $X$ with $ \mathrm{Bel}(X)\neq 0$ we define the visualisation map 
$$\mathrm{Vis}:\mathcal{B}(X) \to [0,1] \times \Delta_X $$
 by
\[
\mathrm{Vis}(\mathrm{Bel})=\Big(     \mathrm{Bel}(X)  , \frac{\sum_{x\in X} \mathrm{Pla}(\{x\})\,e_x}{\sum_{x\in X} \mathrm{Pla}(\{x\}))}     \Big)   \ .
\]
\end{definition}

%%%%%%%%%%%%%%%%%%%%%%%%%%%%%%%%%%%%%%%%%%%%%%%%%%%%%
%%%%%%%%%%%%%%%%%%%%%%%%%%%%%%%%%%%%%%%%%%%%%%%%%%%%%

\subsection{The meanings and visualisation of belief functions} \label{visbelex}
We shall explain the use of the visualisation function in Definition~\ref{visdef} with the set $X=\{a,b,c\}$. 
The 2-simplex $\Delta_X$ is the filled in triangle in Figure~\ref{plasimp}, and the face 
$\Delta_Y$ for $Y=\{a,b\} \subset X$ is shown by the thickened line. 

The first component of the visualisation function can be interpreted as the reliability of the whole procedure. For example, before designers embark on a system, they may evaluate how much confidence they have in the methods and resources they are asked to use. This confidence level $\mathrm{Bel}(X)\in[0,1]$ affects the interpretation of everything which follows. It would be common to use a colour code where green would denote $\mathrm{Bel}(X)\cong 1$, orange an intermediate value and red for a low value of $\mathrm{Bel}(X)$. An alternative would be to attach a numerical value to the point in the simplex given by the second component of the visualisation.
Because of problems with colours and to give a more exact value we use both methods in Figure~\ref{plasimp}, where coloured dots give the positions in the simplex and in addition the value of 
$\mathrm{Bel}(X)$ is given as a percentage.

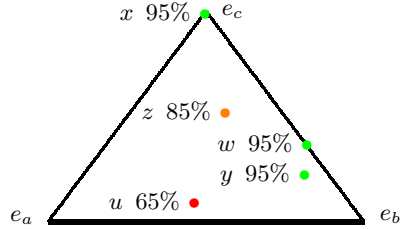
\begin{figure}[htbp]
\begin{center}

\unitlength 0.7 mm
\begin{picture}(85,45)(0,18)
\linethickness{0.3mm}
\multiput(10,20)(0.12,0.16){250}{\line(0,1){0.16}}
\linethickness{0.3mm}
\multiput(40,60)(0.12,-0.16){250}{\line(0,-1){0.16}}
\linethickness{0.6mm}
\put(10,20){\line(1,0){60}}
\linethickness{0.3mm}
%\multiput(100,20)(0.6,-0.12){83}{\line(1,0){0.6}}
\linethickness{0.3mm}
%\multiput(150,10)(0.12,0.24){83}{\line(0,1){0.24}}
\linethickness{0.1mm}
%\multiput(100,20)(0.72,0.12){83}{\line(1,0){0.72}}
\linethickness{0.3mm}
%\multiput(100,20)(0.12,0.16){250}{\line(0,1){0.16}}
\linethickness{0.3mm}
%\multiput(130,60)(0.12,-0.12){250}{\line(1,0){0.12}}
\linethickness{0.3mm}
%\multiput(130,60)(0.12,-0.3){167}{\line(0,-1){0.3}}
%\put(130,64){\makebox(0,0)[cc]{$e_\emptyset$}}

\put(5,20.8){\makebox(0,0)[cc]{$e_a$}}

%\put(95,17){\makebox(0,0)[cc]{$e_a$}}

\put(75,21){\makebox(0,0)[cc]{$e_b$}}

%\put(153,8){\makebox(0,0)[cc]{$e_b$}}

\put(45,60){\makebox(0,0)[cc]{$e_c$}}

%\put(164,30){\makebox(0,0)[cc]{$e_c$}}

\put(33,60){\makebox(0,0)[cc]{   $x$\, 95\%   {\color{green}$\bullet$      } }}
\put(52,35){\makebox(0,0)[cc]{   $w$\, 95\%   {\color{green}$\bullet$      } }}
\put(52,29){\makebox(0,0)[cc]{   $y$\, 95\%   {\color{green}$\bullet$      } }}
\put(37,41){\makebox(0,0)[cc]{   $z$\, 85\%   {\color{orange}$\bullet$      } }}
\put(31,24){\makebox(0,0)[cc]{   $u$\, 65\%   {\color{red}$\bullet$      } }}
%\put(44,25){\makebox(0,0)[cc]{$x$}}

%\linethickness{0.1mm}
%\multiput(112.5,17.5)(0.12,0.12){46}{\line(1,0){0.12}}
%\linethickness{0.1mm}
%\multiput(121.88,15.62)(0.12,0.14){68}{\line(0,1){0.14}}
%\linethickness{0.1mm}
%\multiput(131.88,13.75)(0.12,0.16){83}{\line(0,1){0.16}}
%\linethickness{0.1mm}
%\multiput(142.5,11.25)(0.12,0.18){96}{\line(0,1){0.18}}
%\put(129,44){\makebox(0,0)[cc]{$\bullet$}}
%\put(133,44){\makebox(0,0)[cc]{$z$}}

\end{picture}

\setlength{\belowcaptionskip}{-15pt}
\medskip
\caption{Values of the visualisation map for the set $X=\{a,b,c\}$}
\label{plasimp}
\end{center}
\end{figure}

First, we consider the three green dots in Figure~\ref{plasimp}  with a belief value of $\mathrm{Bel}(X)=0\cdot 95$. We suppose that this value is sufficiently high that we have reasonable confidence in the rest of the values. The point $w$ is on the face $\Delta_{ \{b,c\}  }$ so, according to Proposition~\ref{vischeck}, in this case 
$$\mathrm{Bel}(\{b,c\} )=\mathrm{Bel}(X) \ \textrm{and} \ \mathrm{Bel}(\{a\} )=0.$$ 
The point $x$ is on the face $\Delta_{ \{c\}  }$ so, according to Proposition~\ref{vischeck}, in this case 
$$\mathrm{Bel}(\{c\} )=\mathrm{Bel}(X) \ \textrm{and} \ \mathrm{Bel}(\{a,b\} )=0.$$

The point $y$ is close to the face $\Delta_{ \{b,c\}  }$ so $\mathrm{Bel}(\{a\} )$ is small. In practice, we could decide that $a\in X$ was not needed if $\mathrm{Bel}(\{a\} )$ was sufficiently small, it would be almost impossible to obtain a value which was exactly zero. 

The orange point $z$ corresponds to a belief function for which $\mathrm{Bel}(X)=0\cdot 85$. At this level we suppose that there is significant concern that the model based on $X$ is not sufficient, and may not be performing optimally.
(Of course the threshold levels for concern would be set at different values for different models and problems.)  Still, according to the position being not close to the faces, we see that there is no evidence that we can ignore any elements of $X$. 

The red point $u$ corresponds to a belief function for which $\mathrm{Bel}(X)=0\cdot 65$. At this level we suppose that there is great concern that the model based on $X$ is not sufficient, and may be giving dubious values.

%%%%%%%%%%%%%%%%%%%%%%%%%%%%%%%%%%%%%%%%%%%%%%%%%%%%%
%%%%%%%%%%%%%%%%%%%%%%%%%%%%%%%%%%%%%%%%%%%%%%%%%%%%%
%%%%%%%%%%%%%%%%%%%%%%%%%%%%%%%%%%%%%%%%%%%%%%%%%%%%%
%%%%%%%%%%%%%%%%%%%%%%%%%%%%%%%%%%%%%%%%%%%%%%%%%%%%%

%%%%%%%%%%%%%%%%%%%%%%%%%%%%%%%%%%%%%%%%%%%%%%%%%%%%%
%%%%%%%%%%%%%%%%%%%%%%%%%%%%%%%%%%%%%%%%%%%%%%%%%%%%%
%%%%%%%%%%%%%%%%%%%%%%%%%%%%%%%%%%%%%%%%%%%%%%%%%%%%%
%%%%%%%%%%%%%%%%%%%%%%%%%%%%%%%%%%%%%%%%%%%%%%%%%%%%%

\section{Changing frames of belief and theories} \label{belief_theory_change}
So far, we have been considering systems in which we suppose the modes are chosen and fixed in the process of designing the system. The focus is on how we move between these modes according to prescribed conditions and algorithms that interpret monitoring data. Recalling our anthropomorphic guidelines in Section~\ref{explainability}, and the simple principles about the ideas of innovation, imagination and testing, we can ask: \textit{How might modes be modified or even created in response to unforeseen situations?} This is an open problem but one worthy of some light speculation, given its relevance for our approach. The light speculation soon encounters ancient and deep philosophical issues, however, especially in the epistemologies of the sciences, physical and social.

%%%%%%%%%%%%%%%%%%%%%%%%%%%%%%%%%%%%%%%%%%%%%%%%%%%%%
%%%%%%%%%%%%%%%%%%%%%%%%%%%%%%%%%%%%%%%%%%%%%%%%%%%%%

\subsection{Frames of belief as modes}\label{changing_theories}

\begin{quote}
In so far as a scientific statement speaks about reality, it must be falsifiable: and in so far as it is not falsifiable, it does not speak about reality.

\rightline{\textit{Karl R.\ Popper, The Logic of Scientific Discovery} \cite{PopperSciDis}}
\end{quote}   

Consider how the components of a mode in a system (see Definition~\ref{mode_working definition}) correspond with \textit{models of a real situation}, broadly conceived and interpreted, in the table:

\begin{center}
  \begin{tabular}{@{} |c|c| @{}}
    \hline
    System &  Model \\ 
    \hline
    Algorithm & Theories and formal descriptions to process \\ 
     & results and predict observations      \\  \hline
    Data & Data, determined by the theories, \\
    & from previous observations and computations \\ 
    & e.g.,\ constants of nature, parameters, ... \\ \hline
    Objectives & Predict observations, ... \\ \hline
    Oracles & Experimental apparatus, instruments, tools  \\
    & for making observations, and sampling... \\ \hline
    Visualisation of the & Means to evaluate how well the model predicts observations. \\ 
   belief function.\\ 
    \hline
  \end{tabular}
\end{center}

The visualisation of belief is one of the primary ideas of this paper (Section~\ref{visbelex}), and is focussed on understanding, accounting for and explaining a system; it is also valuable in thinking about design, e.g.,  could a simpler model work?  The question now is: 

\textit{On what foundations do the models and the beliefs depend?} 

We will use the terminology of a \textit{frame of belief} to denote whatever `knowledge'   the models and belief rest upon. Here is a working definition:

\begin{definition}\label{frame_of_belief}
Suppose that $\mathcal{M}$ is a list of theories, laws, principles, axioms, theorems, models, heuristics and guidelines, etc. that may be relevant to a situation; we call these {\em theoretical elements}. A subset $X\subset \mathcal{M}$ we will call a {\em frame of belief}, where to the description of each theoretical element we add methods of interaction between them.
 For modelling a particular problem we likely have to consider only a (small) subset $Y\subset X$ of the contents of $\mathcal{M}$. 
\end{definition}  

In modelling, and especially in modelling social contexts, there is a frame of belief upon which any particular model is built. In science and engineering, we expect the frames of belief to be clearly identifiable and indeed known, if not always forming a concensus.  

In social contexts, the frames of belief we expect to be harder to identify and in many cases contested.  In particular, in social affairs, commonly there is a cocktail of social theory, social and cultural orthodoxies, religious traditions, formal regulation, received opinion, political influence, prejudice, bias, etc.

%-- as one sees commonly in some discussions in which race and gender are introduced as relevant.

\textit{What can be said about a frame of belief in general?} 

Here there is something of a cliff edge for authors of a humble technical paper on systems. The intellectual ocean that is the philosophy of knowledge is in sight with its schools of thought, philosophical theories, and profound debates, which cover so many physical, social and individual concerns. 

From our point view, we see and propose an abstraction:

\medskip
\begin{description}[noitemsep,topsep=0pt,align=left]
\item [Frame of Belief 1: Form.] A frame of belief can be considered as a mode
\item [Frame of Belief 2: Change.] A change of frame of belief can be considered as a change of mode.
\end{description}
\medskip

This means that the set of principles for modes we introduced earlier, in Section \ref{modes}, may be relevant to shaping and developing the notion, e.g., there are basic frames of belief, joint frames of belief, and frames of belief have hierarchical structure.

\subsection{Forms of frames of belief in science}\label{forms_of_frames}

However, in our context of thinking about human scenarios, there is something to gained initially for reflecting on changes in our scientific understanding of the physical world. In science and engineering it is easy to see ideas of innovation, imagination and testing as drivers of problem solving. So, for inspiration, let us turn to the history and philosophy of science and engineering, where we may find some suggestive bits and pieces to illustrate what `knowledge' underpins the making of a mode.

The situation in physical science is philosophically simpler than in human science: 

\begin{example} \label{sciDom}
\textit{Simplicial complex}.  Suppose that $\mathcal{M}$ is a list of scientific theories, old and new, e.g.,\ general relativity, quantum mechanics, atomic theory, Newtonian mechanics, thermodynamics, the luminiferous aether, the fluid theory of heat, the kinetic theory of heat, etc.  At some period in history, we can say that human knowledge has a frame of belief $X$  --  i.e., is in mode $X\subset \mathcal{M}$ --  if the science in $X$ is commonly applied to our understanding of (an aspect of) the natural world. We can take the abstract simplicial complex $\mathcal{C}$ to be the set of subsets $X$ which do not contain immediate contradictions (e.g.,\ Flat Earth and Round Earth theories do not both belong to an element of $\mathcal{C}$). \qed
\end{example}

A mode or model formed of several elements of $\mathcal{M}$ is not just the sum of the different components. 

\begin{example} \label{sciDom}
\textit{Joint modes}.  Consider combining Maxwell's equations of electromagnetism with Newtonian gravity. We need to decide how the theories affect each other. For example, we could decide that electromagnetic fields are not sources of gravity, and that they are in turn not affected by gravitational fields. Of course, this would run into problems because it could not predict gravitational lensing. 
\end{example}

\begin{example} \label{atom}
\textit{Hierarchy of modes}. We can illustrate frames of belief by opening up a mode `atomic theory' into various submodes, corresponding to different ideas of atomic structure \cite{atom1,Justi_Gilbert2000}. From ideas by the classical Greeks (including the name \textit{atom}), John Dalton proposed an idea of atomic theory, but without any internal structure of the atoms. In 1867, William Thomson proposed that atoms were knots in the aether. This was followed in 1878 by Alfred Mayer's proposition of atoms as magnets arranging themselves into molecules. After the discovery of the electron William Thomson and J.J. Thompson (the discoverer of the electron) proposed the `plum pudding' model for the atom. Rutherford proposed a nucleus surrounded by electrons. Niels Bohr introduced the quantisation of the orbits of the electrons, and Schrodinger gave the first modern treatment of atoms using wave functions, which has since been modified by the introduction of relativity and spin (the Klein-Gordon and Dirac equations). \qed
\end{example}

Scientific examples, old and new, are readily available and can be quite specific, with the theoretical elements being precise laws or equations. A major contemporary example would be the competing theories for explaining dark matter \cite{Fore 2020}.

However, for us it is important to note that Definition \ref{frame_of_belief} needs to be explored in terms of human-centred theoretical elements: sociological theories, social categories, economic models, cultural practices, regulations and laws, priorities and bias, social media tags, etc.  The epistemic space of social science is \textit{much} more complicated, of course, as many working notions are imprecise, debated and are culture and country specific. Yet, there is no shortage of government regulations, formal procedures, and notions of `best practice'  that require precision and agreement at some level. Indeed it is these that invite automation.

For our general speculative discussion here, the social examples would contain too many distractions.

However, the attempt to define frames of belief for human and social contexts is directly relevant to our research agenda here, e.g., consider frames of belief as envelopes for the examples in Section \ref{social_examples}.

%%%%%%%%%%%%%%%%%%%%%%%%%%%%%%%%%%%%%%%%%%%%%%%%%%%%%
%%%%%%%%%%%%%%%%%%%%%%%%%%%%%%%%%%%%%%%%%%%%%%%%%%%%%

\subsection{Changing frames of belief in science}\label{changing_theories}

Now we need to look at how frames of belief change. Interpreted as modes, change is facilitated by transition functions, and there is an interesting link between transition functions and information hiding. Two types of change have been highlighted in Section \ref{actions}.

We consider moving from a mode $X$ to a subset mode $Y\subset X$ when we no longer believe that its complement $X\setminus Y$ in $X$ is relevant and may be hidden or discarded. In agreement with Ockham's Razor 2 in Section \ref{actions}, we established a principle to move to a simpler (subset) mode where possible. 

\begin{example} \label{Phlogiston}
Phlogiston was a substance whose release was supposed to cause combustion \cite{PhlogistonWiki}. Over time the existence of Phlogiston became simply not necessary to explain phenomena. In the notation above from mode $X$ we have $X\setminus Y=\{\mathrm{Phlogiston}\}$. \qed
\end{example}

Alternately, consider moving from mode $X$ to a superset mode 
 $X\subset Z$. This is more difficult as we are moving into an unknown country where there are factors in mode $Z$ that mode $X$ was never designed to understand. 
Firstly, we would not want to move from mode $X$ if we believe that it is doing well, there has to be a motivation to move, such as problems. In science, change has come about from `a state of growing crisis': 

\begin{quote}
Thermodynamics was born from the collision of two existing nineteenth-century physical theories, and quantum mechanics from a variety of difficulties surrounding black-body radiation, specific heats, and the photoelectric effect. Furthermore, in all these cases except that of Newton the awareness of anomaly had lasted so long and penetrated so deep that one can appropriately describe the fields affected by it as in a state of growing crisis.  

\rightline{\textit{Thomas S.\ Kuhn, The Structure of Scientific Revolutions \cite{StrSciRev}}}
\end{quote}

However, not so obviously, there is a problem when this crisis is too large.  It is not that we cannot change mode in a time of great crisis (indeed there may be little alternative), it is rather that we would be so far into an unknown country that the result of doing so would not be predictable and beset with risk. 

%\begin{example} \label{GRdisc}
%Initially the experimental evidence for General Relativity was found in what seemed minor details, the precession of the orbit of Mercury and the bending of light around the sun. However, we might imagine a more drastic introduction where a space ship (taking the liberty of assuming that space travel came before General Relativity) found itself in the vicinity of a black hole. The astronauts now find themselves in a situation where almost every observation is altered by gravitational effects and time dilation - a truly unknown country. Their first guesses as to what is going on and extrapolations from previously known physics are likely to be wildly inaccurate. \qed
%\end{example}

Finally, recall Example \ref{sciDom} with visualisation in mind:

\begin{example}\label{queen}
For a philosophical example, we take $\mathcal{M}$ to be the theories of the universe. A set of theories $X\subset \mathcal{M}$ will be a simplex if those theories could possibly co-exist simultaneously, i.e.,\ are not contradictory. Introducing the geometry of simplicies, this means there is no edge between the vertices `flat Earth' and `round Earth', but there is an edge between `special relativity' and `thermodynamics'.  Further, being on that edge would allow more theory and predictions of observations than either theory could in isolation. 
Evidence from observations and experiments would be used to place our actual universe  (from a hypothetical `state space' of universes) at some point in the simplicial complex, or to show that our theories are faulty by awkwardly pointing at a state not compatible with our construction. \qed
\end{example}

\begin{example}\label{xcxcxc}
In particle physics we have modes consisting of field theories, so Quantum Electrodynamics would be a mode and the Standard Model would be another mode. We begin in such a mode (currently the Standard Model), and experiments are proposed and modelled (typically by Feynman diagrams). Then data from actual experiments is compared to the calculated values. If the experimental results differ from the calculated values by $3\sigma$ ($\sigma$ denotes one standard deviation in terms of experimental and statistical error) then we are concerned that the model may not be correct.  If the experimental results differ from the calculated values by $5\sigma$ then we believe that there is a real problem and seek a solution by modifying the model. This would typically be by proposing a new particle or field to add to the model. 
\end{example}

%\smallskip\noindent
%\textbf{What transitions make sense?}  
%So far we have considered mode transition to a superset and a subset respectively. What about more general mode transitions? In moving from a worldview dominated by `magic' (for want of a better term) to one dominated by the scientific method humans made many small steps. There was a large continuity of understanding in many areas while other areas evolved rapidly, for example alchemy evolving into modern chemistry or the Copernican revolution. To move in one jump from a stone age understanding of the world to a modern understanding would be next to impossible, as the discontinuity would make transfer of information almost impossible. We can theorise that changes of mode should take place with certain factors in common between the old and new modes.
%

%%%%%%%%%%%%%%%%%%%%%%%%%%%%%%%%%%%%%%%%%%%%%%%%%%%%%
%%%%%%%%%%%%%%%%%%%%%%%%%%%%%%%%%%%%%%%%%%%%%%%%%%%%%
%%%%%%%%%%%%%%%%%%%%%%%%%%%%%%%%%%%%%%%%%%%%%%%%%%%%%
%%%%%%%%%%%%%%%%%%%%%%%%%%%%%%%%%%%%%%%%%%%%%%%%%%%%%

\section{Concluding remarks}

%%%%%%%%%%%%%%%%%%%%%%%%%%%%%%%%%%%%%%%%%%%%%%%%%%%%%
%%%%%%%%%%%%%%%%%%%%%%%%%%%%%%%%%%%%%%%%%%%%%%%%%%%%%

\subsection{Reflections}\label{reflections}

Our paper is about automation in the context of human-centred systems. With automation in mind, we have introduced 

\medskip
1. a new general system model based on a notion of modes and mode transitions;

2. formalised the model mathematically both algebraically and geometrically;

3.  advocated an anthropomorphic approach that focuses on the explication of systems, based on intuitions about conceptual frameworks and frames of belief; 

4. proposed human-centred case studies to illustrate/test the model; and

5. connected our concept of `belief about data' with the established quantitative theory of belief pioneered by Dempster and Shafer.

\medskip
With purely technical ideas in mind, the contributions are (1), (2) and (5), introducing concepts and data structures for modelling modes in which the position of the system in the geometric visualisation of the modes represents quantified judgements or beliefs about the state of the system. Thus, our system model focuses on the dynamics of changing interpretations of monitoring data and offers some insight into reasons for decisions taken to change the behaviour of the system.

With human-centred applications in mind, the contributions are (3) and (4) in which we align the system model with simple working assumptions and principles of human reasoning and scenarios for automation.   To guide the design of changes in the system we considered, for example, Incompleteness and Occam's Razor (as described in Section~\ref{explainability}), which motivate a move to a simpler system when possible, and to a more complicated system only when forced to do so. We noted how changes in the `frames of belief' that underly the design of modes are also a factor, and we speculated that they can have the same structure as modes. 

\subsection{Questions and answers}\label{Q&A}
Here are some questions that come to mind that may help to clarify a few points:
 
\smallskip
\noindent\textbf{Q:} \textit{Does the modes model apply only to the highest level of the organisation of the system?} 

\noindent\textbf{A:} No, we can have a hierarchy where each of the modes in the top level becomes, on closer examination, a simplicial complex of modes (called sub-modes, and having the same organising principles). From the system design point of view, this is simply a form of information hiding (e.g., a single mode labelled `court' for a police officer opens into a multitude of submodes for those organising court proceedings). From a mathematical point of view, the relevant idea to preserve the structure is a simplicial map between simplicial complexes. 

\smallskip
\noindent\textbf{Q:} \textit{How do modes relate to the environment?} 

\noindent\textbf{A:} 
For dealing mathematically with the real world we have in mind the formalism of \textit{physical oracles} \cite{Axiomatising}  that both gather information and affect the real world, and allows for errors and delays. Each mode has a set of such oracles with which it can communicate. 

 \smallskip
\noindent\textbf{Q:} \textit{How does the geometry specify the algorithms?} 

\noindent\textbf{A:} There is a huge amount of information hiding in which the geometry of the modes specify which modes neighbour the current mode, and the algorithm only knows about those neighbouring modes.

 \smallskip
\noindent\textbf{Q:} \textit{Can a system of modes be retro-fitted to explain some system?} 

\noindent\textbf{A:} Perhaps: because modes combine objectives, specifications and implementations, this combination can identify the scope and limits of components and refactor a complex system.

\smallskip

Modes have a number of connections with general techniques and themes in computing:

 \smallskip
\noindent\textbf{Q:} \textit{How do modes relate to learning algorithms such as neural nets?} 

\noindent\textbf{A:} 
The oracle formalism can also be used to import output from a potentially non-explainable algorithm (such as a neural net), where the possible error can quantified and considered along with other information. Actually, the output of a classification procedure, such as a neural net, can be taken as an element of a simplex spanned by the elements of the set of possible classifications, and thus can fit into a simplicial approach.

\smallskip
\noindent\textbf{Q:} \textit{How do modes relate to general themes such as security, safety and privacy?} 

\noindent\textbf{A:} 
Modes can be classified by allocated a label for security clearance, potential harm or  sensitivity. This can be done via a simplicial map to a simplicial complex of appropriate classifications and this can determine what information (via the oracles) each mode has access to. Only certain modes would have access to oracles allowing some possibly problematic behaviour -- other modes simply cannot access these oracles.  For instance, simplicial complexes stand comparison with lattices \cite{dennInfo} for a formal study of security. Oracles cover all external computers and databases and sensors and actuators . Judgements can be made as to how reliable these oracles are.

%%%%%%%%%%%%%%%%%%%%%%%%%%%%%%%%%%%%%%%%%%%%%%%%%%%%%
%%%%%%%%%%%%%%%%%%%%%%%%%%%%%%%%%%%%%%%%%%%%%%%%%%%%%

\subsection{On systems in AI and current controversies}\label{current_controversy}

Considering the list of five contributions of this paper above, numbers (1), (2) and (5) are about general system theories, which are relevant to the design of many kinds of computing system. The problem of controlling autonomous systems in unpredictable environments is central to the remit of AI, which has benefitted from new system models with different motivations, features and mathematical theories. Since the 1960s many new system  models have been developed such as deterministic and nondeterministic networks of computing units and concurrent systems shaped by theories of automata, probability and physical systems as in control systems \cite{Kalmanetal1969} and hybrid systems \cite{MoBiAI}. Our `first principles approach' to the technical material aligns with more fundamental conceptual approaches to systems, such as Carl Hewitt's thinking about concurrency and his actor model, which finds a home in AI through the modern and conceptually strong theory of agents \cite{WooldridgeJennings1995,Wooldridge2009,LeonardosPiliouras2022}. The  usefulness of our technical work to aspects of multiagent modelling is something worth investgating. Perhaps agents can be factorised into modes or simply be modes.

Contributions numbered (3) and (4) are relevant to human-centred AI applications, which are the cause of current AI controversies. Issues such as transparency, explainability, etc. are more prominent, significant and difficult to tackle for human-centred AI applications.  In this paper, we consider automation of human-centred tasks. We also try to consider human reasoning and innovation, and their links with modes through examples. 

One underlying theme is an idea of belief, and we use this and simplicial geometry both to make decisions and to visualise beliefs about the state of the system, both these being important topics in AI. The simplicial model is a higher dimensional data structure whose compactess and natural heirachical structure improves on planar graphs.  Graphs of various kinds can be found in most system theories, and play a role in explainability, e.g., in machine learning \cite{TiddiSchlobach}.

The frameworks of belief underlying human-centred systems are hard to analyse and codify even for our simplified examples and seem destined to become a topic in their own right. Our refections are shaped by the role of models and theories in the history of science, which are convenient, general and better understood than examples of frameworks of belief that are highly contextual, made from from social theories, regulations, cultural norms, or best practice. 

Thus, we see our paper as offering some conceptual and mathematical tools for thinking about AI systems to better understand and make them more more easily accountable, which are essential as AI systems deal with people.

Any discussion of methods for understanding automation may be relevant to currently growing AI controversies. For example, among a number of headline grabbing actions an open letter ``Pause Giant AI Experiments" has been signed by many leading industrialists and academics.\footnote{See: https://futureoflife.org/open-letter/pause-giant-ai-experiments.}  In the space of a few months, this was followed by further warnings, international meetings about regulation, widespread debates on AI tools in many sectors. Our connection to these `growing crises' is through our interest in explainable automations, most urgently in decision making in human affairs.

\medskip
\noindent \textit{Legacy liabilities}. We are proposing an explanatory framework for logical AI based on human reasoning, which would be used to contain and limit machine learning (e.g.,\ by imposing Asimov's laws and other social norms). Sufficiently extensive interconnected logical AIs, developed over decades with many technologies, are bound to pose problems in the years to come. If legacy AIs
were also granted permission to innovate and mutate, the problem could be intractable -- given our problems with current legacy software.  It is rather essential that any means by which we can understand, impose limits on, reason with and keep track of, etc, an automation is explored. And it helps to know what we might be looking for as we research the future.

\medskip
\noindent \textit{Decision making}.  At the centre of our thinking about modes is decision making. Decision making permeates computing applications. Monitoring and managing systems involves decisions to change modes of operation based on judgements about available monitoring data. In human-centred systems, classifying people's data involves choices about the characteristics that define typologies, from recommender systems in shopping to identifying and tagging objects and faces. In the professional and social realm, decisions based on the information available about people have automatically filtered and processed applications and requests, from creditworthy loans to job applications. Commonly, decisions in applications involve weighing up data, arguments, evidence, etc. and making a judgement. Probability theories have helped analyse decision making since the 18th Century, in the birth of the mathematics of insurance. More recently, belief theories \cite{Shaf,DemSchBel}, many valued logics \cite{MVLog},  and sundry machine learning techniques \cite{Manglaetal2023}, have provided mathematical theories with which to explore inexact reasoning and making decisions. Underwritten by the guiding formalisation of an abstract simplicial complex, the scope of the modes could become much wider.

%%%%%%%%%%%%%%%%%%%%%%%%%%%%%%%%%%%%%%%%%%%%%%%%%%%%%
%%%%%%%%%%%%%%%%%%%%%%%%%%%%%%%%%%%%%%%%%%%%%%%%%%%%%
%%%%%%%%%%%%%%%%%%%%%%%%%%%%%%%%%%%%%%%%%%%%%%%%%%%%%
%%%%%%%%%%%%%%%%%%%%%%%%%%%%%%%%%%%%%%%%%%%%%%%%%%%%%

\bibliographystyle{compj}
%\bibliography{ModellingBidders}

\end{document}